%% file: DistributedPCA_arXiv.tex
\documentclass[twoside,11pt]{article}

\usepackage{times}
\usepackage{fullpage}

\usepackage{graphicx,tikz} 
\usepackage{subfigure}

\usepackage[round]{natbib}

\usepackage{algorithm}
\usepackage{algorithmic}
\usepackage{hyperref}

\usepackage{amsmath,amsthm,amssymb,mathtools}
\usepackage{Definitions}

\theoremstyle{plain}

\newtheorem{claim}{Claim}
\newtheorem{fact}{Fact}

\theoremstyle{definition}
\newtheorem{defn}{Definition}

\theoremstyle{remark}

\newcommand{\littleheader}[1]{\noindent \textbf{#1} }
\newenvironment{proofsketch}{\par\noindent{\bfseries\upshape Proof Sketch:}}{\hfill \qed}

\newtheorem*{rep@theorem}{\rep@title}
\newenvironment{oneshot}[1]{\def\rep@title{#1} \begin{rep@theorem}}{\end{rep@theorem}} 

\newcommand{\x}{\mathcal{L}}
\newcommand{\R}{\RR}
\newcommand{\N}{\NN}
\newcommand{\MP}{\mathbf{P}} 

\newcommand{\MA}{\mathbf{A}}
\newcommand{\MB}{\mathbf{B}}
\newcommand{\ME}{\mathbf{E}}
\newcommand{\MU}{\mathbf{U}}
\newcommand{\MD}{\mathbf{\Sigma}}
\newcommand{\MV}{\mathbf{V}}
\newcommand{\MX}{\mathbf{X}}
\newcommand{\MY}{\mathbf{Y}}
\newcommand{\MQ}{\mathbf{Q}}
\newcommand{\Mbf}{\mathbf{M}}
\newcommand{\MS}{\mathbf{S}}
\newcommand{\MH}{\mathbf{H}}
\newcommand{\MI}{\mathbf{I}}
\newcommand{\ML}{\mathbf{L}}
\newcommand{\MR}{\mathbf{R}}
\newcommand{\MT}{\mathbf{T}}
\newcommand{\MZ}{\mathbf{Z}}
\newcommand{\MOmega}{\mathbf{\Omega}}
\newcommand{\MPhi}{\mathbf{\Phi}}
\newcommand{\tc}[2]{{#1}^{(#2)}}
\newcommand{\tcd}[1]{\widehat{#1}}
\newcommand{\nnz}{\mathrm{nnz}}
\newcommand{\lspan}[1]{L_{#1}}

\newcommand{\gpca}[1]{\mathbf{\tilde{#1}}}
\newcommand{\gpcap}[1]{\tilde{#1}}

\newcommand{\proj}{\pi}

\newcommand{\nc}{k}
\newcommand{\dimr}{r}

\newcommand{\mybullet}{\rhd}

\renewcommand{\algorithmicrequire}{\textbf{Input:}}
\renewcommand{\algorithmicensure}{\textbf{Output:}}

\begin{document}

\title{Improved Distributed Principal Component Analysis}
\author{
Maria-Florina Balcan\\
School of Computer Science\\
Carnegie Mellon University\\
\texttt{ninamf@cs.cmu.edu}\\
\and
Vandana Kanchanapally\\
School of Computer Science\\ 
Georgia Institute of Technology\\
\texttt{vvandana@gatech.edu}\\
\and 
Yingyu Liang\\
Department of Computer Science\\
Princeton University\\
\texttt{yingyul@cs.princeton.edu}\\
\and
David Woodruff\\
Almaden Research Center\\
IBM Research\\
\texttt{dpwoodru@us.ibm.com}
}

\date{}

\maketitle

\begin{abstract}
We study the distributed computing setting in which there are multiple servers, each holding a set of points, who wish to compute functions on the union of their point sets. A key task in this setting is Principal Component Analysis (PCA), in which the servers would like to compute a low dimensional subspace capturing as much of the variance of the union of their point sets as possible. Given a procedure for approximate PCA, one can use it to approximately solve $\ell_2$-error fitting problems such as $k$-means clustering and subspace clustering. The essential properties of an approximate distributed PCA algorithm are its communication cost and computational efficiency for a given desired accuracy in downstream applications. We give new algorithms and analyses for distributed PCA which lead to improved communication and computational costs for $k$-means clustering and related problems. Our empirical study on real world data shows a speedup of orders of magnitude, preserving communication with only a negligible degradation in solution quality.
%
%
%
Some of these techniques we develop, such as a general transformation from a constant success probability subspace embedding to a high success probability subspace embedding with a dimension and sparsity independent of the success probability, may be of independent interest.
\end{abstract}

\section{Introduction}

Since data is often partitioned across multiple servers~\citep{olston2003adaptive,corbett2012spanner,mitra2011characterizing}, there is an increased interest in computing on it in the distributed model. A basic tool for distributed data analysis is Principal Component Analysis (PCA). The goal of PCA is to find an $\dimr$-dimensional (affine) subspace that captures as much of the variance of the data as possible. Hence, it can reveal low-dimensional structure in very high dimensional data. Moreover, it can serve as a preprocessing step to reduce the data dimension in various machine learning tasks, such as $k$-means, Non-Negative Matrix Factorization (NNMF)~\citep{seung2001algorithms} and Latent Dirichlet Allocation (LDA)~\citep{blei2003latent}. 

In the distributed model, approximate PCA was used by \citet{dan2013tiny} for solving a number of shape fitting problems such as $k$-means clustering, where the approximation PCA solution is computed based on a summary of the data called {\it coreset}.  The coreset has the property that local coresets can be easily combined across servers into a global coreset, which then leads to an approximate PCA solution to the union of the data sets. Designing small coresets therefore leads to communication-efficient protocols. Coresets have the nice property that their size typically does not depend on the number $n$ of points being approximated. A beautiful property of the coresets developed in \citep{dan2013tiny} is that for approximate PCA their size also only depends linearly on the dimension $d$, whereas previous coresets depended quadratically on $d$ \citep{feldman2011unified}. This gives the best known communication protocols for approximate PCA and $k$-means clustering. 

Despite this recent exciting progress, several important questions remain. First, can we improve the communication further as a function of the number of servers, the approximation error, and other parameters of the downstream applications (such as the number $k$ of clusters in $k$-means clustering)? Second, while preserving optimal or nearly-optimal communication, can we improve the computational costs of the protocols? We note that in the protocols of Feldman et al. each server has to run a singular value decomposition (SVD) on her local data set, while additional work needs to be performed to combine the outputs of each server into a global approximate PCA. Third, are these algorithms practical and do they scale well with large-scale datasets?
In this paper we give answers to the above questions. To state our results more precisely, we first define the model and the problems.

In the distributed setting, we consider a set of $s$ servers each of which can communicate with a central coordinator. 
The global data $\MP \in \RR^{n \times d}$, consisting of $n$ points in $d$ dimension, is arbitrarily partitioned on the  servers, where the server $i$ holds $n_i$ points $\MP_i$. The PCA problem is to find an $\dimr$-dimensional subspace which minimizes the sum of the $\ell_2$ distances of the points to their projections on the subspace.

For approximate distributed PCA, the following protocol is implicit in \citep{dan2013tiny}: each server $i$ computes
its top $O(r/\epsilon)$ principal components $\MY_i$ of $\MP_i$ and sends them to the coordinator. The coordinator stacks
the matrices $\MY_i$ on top of each other, forming an $O(sr/\epsilon) \times d$ matrix $\MY$, 
and computes
the top $r$ principal components of $\MY$, and returns these to the servers. This provides a relative-error
approximation to the PCA problem. We refer to this algorithm as Algorithm {\sf disPCA}. 

\littleheader{Our Contributions.}
Our results are summarized as follows.

{\it Improved Communication:} We improve the communication cost for using distributed PCA for $k$-means clustering and similar $\ell_2$-fitting problems. 
The best previous approach is to use Corollary 4.5 in \citep{dan2013tiny}, which shows
that given a data matrix $\MP$, if we project the rows onto the space spanned by the top $O(k/\epsilon^2)$ principal components,
and solve the $k$-means problem in this subspace, we obtain a $(1+\epsilon)$-approximation. In the distributed setting,
this would require first running Algorithm {\sf disPCA} to compute the $O(k/\epsilon^2)$ global principal components, which has communication $O(skd/\epsilon^3)$.
Then one can solve a distributed $k$-means problem in this subspace, and an $\alpha$-approximation in it 
translates to an overall $\alpha(1+\epsilon)$ approximation.

Our Theorem \ref{thm:coreset_gen} shows that it suffices to run Algorithm {\sf disPCA} while only incurring $O(skd/\epsilon^2)$ communication to compute a set of $O(k/\epsilon^2)$ vectors, which can preserve up to a $(1+\epsilon)$ factor the $k$-means cost (and the cost in similar $\ell_2$ error problems).  This set of vectors, though not the optimal $O(k/\epsilon^2)$ global principal components, enjoy a property called \emph{close projection}.\footnote{The close projection property is as follows. Let $\widetilde{\MP}$ denote the projection of $\MP$ on subspace spanned by these vectors. Then for any subspace $\MX$ of dimension $k$, the projections of $\MP$ and $\widetilde{\MP}$ on $\MX$ are close. See Lemma~\ref{thm:DisPCA_P_l2} for the details.} This property guarantees that the $k$-means cost (and the cost in similar $\ell_2$ error problems) of the data  projected on these vectors approximates that of the original data up to a multiplicative $(1+\epsilon)$ factor, which suffices for solving these downstream applications. 
Our communication is thus a $1/\epsilon$ factor better, and illustrates that for downstream applications it is sometimes important to ``open up the box'' rather than to directly use the guarantees of a generic PCA algorithm (which would give $O(skd/\epsilon^3)$  communication). One feature of this approach is that by using the distributed $k$-means algorithm in \citep{disClustering13} on the projected data, the coordinator can sample points from the servers proportional to their local $k$-means cost solutions, which reduces the communication roughly by a factor of $s$ in the $k$-means step, which would come from each server sending their local $k$-means coreset to the coordinator. 
Furthermore, before applying {\sf disPCA} and distributed $k$-means algorithms, one can first run any other dimension reduction to dimension $d'$ so that the $k$-means cost is preserved up to certain accuracy. For example, if we want a $(1+\epsilon)$ approximation factor, we can set $d'=O(\log n /\epsilon^2)$ by a Johnson-Lindenstrauss transform; if we want a larger $2+\epsilon$ approximation factor, we can set $d' =  O(k/\epsilon^2)$ using~\citep{bzmd11}. In this way the parameter $d$ in the above communication cost  bound can be replaced by $d'$.
Note that unlike these dimension reduction methods, our algorithm for projecting onto principal components is deterministic and does not incur error probability. 

{\it Improved Computation:} We turn to the computational cost of Algorithm {\sf disPCA}, which to the best of our knowledge has not
been addressed. A major bottleneck is that each server is computing a singular value decomposition (SVD)
of its point set $\MP_i$, which takes $\min(n_i d^2, n_i^2 d)$ time. We change Algorithm {\sf disPCA} to instead have each server first
sample an oblivious subspace embedding (OSE) \citep{s06,clarkson2013low,nelson2012osnap,meng2013low} matrix $\MH_i$, and
instead run the algorithm on the point set defined by the rows of $\MH_i \MP_i$. Using known OSEs, one can choose
$\MH_i$ to have only a single non-zero entry per column and thus $\MH_i \MP_i$ can be computed in $\nnz(\MP_i)$ time. Moreover,
the number of rows of $\MH_i$ is $O(d^2/\epsilon^2)$, which may be significantly less than the original $n_i$ number of rows. This
number of rows can be further reducted to $O(d \log^{O(1)} d / \epsilon^2)$ if one is willing to spend $O(\nnz(\MP_i) \log^{O(1)}d/\epsilon)$
time \citep{nelson2012osnap}. 
We note that the number of non-zero entries of $\MH_i \MP_i$ is no more than that of $\MP_i$.

One technical issue is that each of $s$ servers is locally performing a subspace embedding, which succeeds with only constant
probability. If we want a single non-zero entry per column of $\MH_i$, to achieve success probability $1-O(1/s)$ so that we
can union bound over all $s$ servers succeeding, we naively would need to increase the number of rows of $\MH_i$ by a factor
linear in $s$. 
We give a general technique, which takes a subspace embedding that succeeds with constant probability as a black box, and show how to perform a procedure which applies it $O(\log 1/\delta)$ times independently and from these applications finds one which is guaranteed to succeed with probability $1-\delta$. Thus, in this setting the players can compute a subspace embedding of their data in $\nnz(\MP_i)$ time, for which the number of non-zero entries of $\MH_i \MP_i$ is no larger than that of $\MP_i$, and without incurring this additional factor of $s$. This may be of independent interest.

It may still be expensive to perform the SVD of
$\MH_i \MP_i$ and for the coordinator to perform an SVD on $\MY$ in Algorithm {\sf disPCA}. We therefore replace the SVD computation with
a randomized approximate SVD computation with spectral norm error. Our contribution here is to analyze the error in distributed
PCA and $k$-means after performing these speedups. 

{\it Empirical Results:}
Our speedups result in significant computational savings. 
The randomized techniques we use reduce the time by orders of magnitude on medium and large-scal data sets, 
while preserving the communication cost. Although the theory predicts a new small additive error because of our speedups,
in our experiments the solution quality was only negligibly affected. 

\littleheader{Related Work}
A number of algorithms for approximate distributed PCA have been proposed~\citep{qu2002principal,bai2005principal,le2008distributed,macua2010consensus,dan2013tiny}, but either without theoretical guarantees, or without considering communication. 
\citet{qu2002principal} proposed an algorithm but provided no analysis on the tradeoff between communication and approximation. Most closely related to our work is~\citep{dan2013tiny}, which observes that the top singular vectors of the local point set can be viewed as its summary and the union of the local summaries can be viewed as a summary of the global data, i.e., Algorithm {\sf disPCA} discussed above.

In \citep{kannan2013nimble} the authors study algorithms in the arbitrary partition model in which each server holds a matrix $\MP_i$ and $\MP = \sum_{i=1}^s \MP_i$. Thus, each row of $\MP$ is additively shared across the $s$ servers, whereas in our model each row of $\MP$ belongs to a single server, though duplicate rows are allowed. Our model is motivated by applications in which points are indecomposable entities. As our model is a special case of the arbitrary partition model, we can achieve more efficient algorithms. For instance, our distributed PCA algorithms provide much stronger guarantees, see, e.g., Lemma \ref{thm:DisPCA_P_l2}, which are needed for the downstream $k$-means application. 
Moreover, our $k$-means algorithms are more general, in the sense that they do not make a well-separability assumption, and more efficient in that the communication of \citep{kannan2013nimble} is $O(sd^2) + s (k/\epsilon)^{O(1)}$ words as opposed to our 
$O(sdk/\epsilon^2) + sk + (k/\epsilon)^{O(1)}$.  

After the announce of this work, \cite{cohen14dim} improve the guarantee for the $k$-means application in two ways. First, they tighten the result in~\citep{dan2013tiny}, showing that projecting to just the $O(k/\epsilon)$ rather than $O(k/\epsilon^2)$ top singular vectors is sufficient to approximate $k$-means with $(1+ \epsilon)$ error. Second, they show that performing a Johnson-Lindenstrauss transformation down to $O(k/\epsilon^2)$ dimension gives $(1+\epsilon)$ approximation without requiring a $\log(n)$ dependence. This can be used as a preprocessing step before our algorithm, replacing $d$ with $O(k/\epsilon^2)$ in our communication bounds. They further show how to reduce the dimension to $O(k/\epsilon)$ using only $O(sk/\epsilon)$vectors, but by a technique different from distributed PCA.

Other related work includes the recent \citep{gp13} (see also the references therein), who give a deterministic streaming algorithm for low rank approximation in which each point of $\MP$ is seen one at a time and uses $O(dk/\epsilon)$ words of communication. Their algorithm naturally gives an $O(sdk/\epsilon)$ communication algorithm for low rank approximation in the distributed model. However, their algorithm for PCA doesn't satisfy the stronger guarantees of Lemma \ref{thm:DisPCA_P_l2}, and therefore it is unclear how to use it for $k$-means clustering. It also involves an SVD computation for each point, making the overall computation per server $O(n_i d r^2/\epsilon^2)$, which is slower than what we achieve, and it is not clear how their algorithm can exploit sparsity.

Speeding up large scale PCA using different versions of subspace embeddings was also considered in \citep{KM13a}, though not in a distributed setting and not for $\ell_2$-error shape fitting problems. Also, their error guarantees are in terms of the $r$-th singular value gap, and are incomparable to ours.

\section{Preliminaries}\label{sec:form}
\littleheader{Communication Model.}
In the distributed setting, we consider a set of $s$ nodes $\Vcal=\{v_i, 1\leq i\leq
s\}$, each of which can communicate with a central coordinator
$v_0$.
On each node $v_i$, there is a local data matrix $\MP_i \in \R^{n_i\times d}$ having $n_i$ data points in $d$ dimension ($n_i > d$).
The global data $\MP \in \R^{n \times d}$ is then a concatenation of the local data matrix, i.e.\
$\MP^\top = \left[\MP_1^\top, \MP_2^\top, \dots, \MP_s^\top\right]$ and $n=\sum_{i=1}^s n_i$.
Let $p_i$ denote the $i$-th row of $\MP$.
Throughout the paper, we assume that the data points are centered to have zero mean,
i.e., \ $\sum_{i=1}^n p_i = 0$.
Uncentered data requires a rank-one modification to the algorithms, whose communication and computation costs are dominated
by those in the other steps.

\littleheader{Approximate PCA and $\ell_2$-Error Fitting.}
For a matrix $\MA=[a_{ij}]$, let $\|\MA\|^2_F = {\sum_{i,j}a_{ij}^2}$ be its Frobenius
norm, and let $\sigma_i(\MA)$ be the $i$-th singular value of $\MA$.
Let $\tc{\MA}{t}$ denote the matrix that contains the first $t$ columns of $\MA$.
Let $\lspan{\MX}$ denote the linear subspace spanned by the columns of $\MX$.
Note that for an orthonormal matrix $\MX$,
the projection of a point $p$ to $\lspan{\MX}$ will be $p\MX$ using the coordinates with respect to the column space of $\MX$, and will be $p\MX\MX^\top$ using the original coordinates.
Let $\proj_L(p)$ be its projection onto subspace $L$
and let $\proj_\MX(p)$ be shorthand for $\proj_{\lspan{\MX}}(p) = p\MX\MX^\top$.

For a point $p \in \R^d$ and a subspace $L \subseteq \R^d$, we denote
the squared distance between $p$ and $L$ by
$$d^2(p,L) := \min_{q\in L} \|p-q\|_2^2 = \|p-\proj_L(p)\|_2^2.$$


%
\begin{defn}
The linear (or affine) $\dimr$-Subspace $\nc$-Clustering on  $\MP \in \R^{n \times d}$ is
\begin{eqnarray}
\min_{\x } d^2(\MP, \x) := \sum_{i=1}^n \min_{L \in \x} d^2(p_i,L)\label{pro:subspace_clustering}
\end{eqnarray}
where $\MP$ is an $n\times d$ matrix whose rows are $p_1, \dots, p_n$,
and $\x=\{L_j\}_{j=1}^{\nc}$ is a set of $\nc$ centers, each of which is an $\dimr$-dimensional linear (or affine) subspace.
\end{defn}
PCA is a special case when $\nc=1$ and the center is an $\dimr$-dimensional subspace. It is well known that the optimal $\dimr$-dimensional subspace is spanned by the top $r$ eigen-vectors of the covariance matrix $\MP^\top \MP$, also known as the principal components. Equivalently, these vectors are the right singular vectors of $\MP$,  and can be found using the singular value decomposition (SVD) on $\MP$.  

Another special case of $\dimr$-Subspace $\nc$-Clustering is $k$-means clustering when the centers are points ($\dimr=0$).
Constrained versions of this problem include NNMF where the $\dimr$-dimensional subspace should be spanned by positive vectors,
and LDA which assumes a prior distribution defining a probability for each $\dimr$-dimensional subspace.
We will primarily be concerned with relative-error approximation algorithms, for which we would like to output a set
$\x'$ of $k$ centers for which $d^2(\MP, \x') \leq (1+\epsilon) \min_{\x} d^2(\MP, \x)$.

\section{Tradeoff between Communication and Solution Quality}\label{sec:review}

Algorithm {\sf disPCA} for distributed PCA is suggested in~\citep{qu2002principal,dan2013tiny}, which consists of a local stage and a global stage. In the local stage, each node performs SVD on its local data matrix, and communicates the first $t_1$ singular values $\tc{\MD_i}{t_1}$ and the first $t_1$ right singular vectors $\tc{\MV_i}{t_1}$ to the central coordinator. Then in the global stage, the coordinator concatenates $\tc{\MD_i}{t_1}(\tc{\MV_i}{t_1})^\top$ to form a matrix $\MY$, and performs SVD on it to get the first $t_2$ right singular vectors.
See Algorithm~\ref{alg:disPCA} for the details and see Figure~\ref{fig:algo} for an illustration. 

\begin{algorithm}[t]
\caption{Distributed PCA algorithm {\sf disPCA}}
\label{alg:disPCA}
\begin{algorithmic}[1]
\REQUIRE{local data $\{\MP_i\}_{i=1}^s$ and parameter $t_1, t_2 \in \N_+$.}
\FOR{each node $v_i \in \Vcal$ }
    \STATE{Compute local SVD: $\MP_i = \MU_i \MD_i \MV_i^\top$.}
    \STATE{Send $\tc{\MD_i}{t_1}, \tc{\MV_i}{t_1}$ to the central coordinator.}
    \STATE{\quad $\mybullet$ {dimension: $[\MP_i]_{n_i\times d}, [\MU_i]_{n_i\times n_i}, [\MD_i]_{n_i\times d}, [\MV_i]_{d\times d}, [\tc{\MD_i}{t_1}]_{n_i\times t_1}, [\tc{\MV_i}{t_1}]_{d\times t_1}$} }
\ENDFOR
\FOR{the central coordinator}
    \STATE{Set $\MY_i= \tc{\MD_i}{t_1}(\tc{\MV_i}{t_1})^\top$, $\MY^\top = [\MY_1^\top,\dots,\MY_s^\top]$.}
    \STATE{Compute global SVD: $\MY = \MU \MD \MV^\top$.}
    \STATE{\quad $\mybullet$ {dimension: $[\MY_i]_{n_i\times d}, [\MY]_{n\times d}, [\MU]_{n\times n}, [\MD]_{n\times d}, [\MV]_{d\times d}, [\tc{\MV}{t_2}]_{d\times t_2}$} }
\ENDFOR
\ENSURE{$\tc{\MV}{t_2}$.}
\end{algorithmic}
\end{algorithm}

\begin{figure}[t]
\begin{center}
\begin{eqnarray*}
\MP = 
\left [
\begin{array}{c}
\MP_1 \\
\vdots \\
\MP_s
\end{array} 
\right ]
\begin{array}{c}
\xrightarrow{\textrm{Local PCA}} \\
\vdots \\
\xrightarrow{\textrm{Local PCA}}
\end{array} 
\left [
\begin{array}{c}
\MD^{(t_1)}_1 \rbr{\MV_1^{(t_1)}}^\top \\
\vdots \\
\MD^{(t_1)}_s \rbr{\MV_s^{(t_1)}}^\top
\end{array} 
\right ]
=
\left [
\begin{array}{c}
\MY_1 \\
\vdots \\
\MY_s
\end{array} 
\right ]
= \MY \xrightarrow{\textrm{Global PCA}} \MV^{(t_2)}
\end{eqnarray*}
\end{center}
\caption{
{
The key points of the algorithm {\sf disPCA}. 
}
}\label{fig:algo}

\end{figure}

To get some intuition, consider the easy case when the data points actually lie in an $\dimr$-dimensional subspace. We can run Algorithm {\sf disPCA}  
with $t_1=t_2=\dimr$. Since $\MP_i$ has rank $\dimr$, its projection to the subspace spanned by its first $t_1=\dimr$ right singular vectors,  $\tcd{\MP}_i = \MU_i \tc{\MD_i}{\dimr} (\tc{\MV_i}{\dimr})^\top$,  is identical to $\MP_i$. Then we only need to do PCA on $\tcd{\MP}$, the concatenation of $\tcd{\MP}_i$. Observing that $\tcd{\MP} = \widetilde{\MU} \MY$ where $\widetilde{\MU}$ is orthonormal, it suffices to compute SVD on $\MY$, and only $\tc{\MD_i}{\dimr} \tc{\MV_i}{\dimr}$ needs to be communicated.  
In the general case when the data may have rank higher than $\dimr$, it turns out that one needs to set $t_1$ sufficiently large, so that $\tcd{\MP}_i$ approximates $\MP_i$ well enough and does not introduce too much error into the final solution. In particular, the following \emph{close projection} property about SVD is the key for the analysis:

\begin{lemma}\label{lem:pca_app}
Suppose $\MA$ has SVD $\MA=\MU \MD \MV$ and let $\tcd{\MA} = \MA \tc{\MV}{t} (\tc{\MV}{t})^\top$ denote its SVD truncation. 
If $t=O(\dimr/\epsilon)$, then for any $d \times \dimr$ matrix $\MX$ with orthonormal columns,
\begin{eqnarray*}
0 \leq \|\MA \MX - \tcd{\MA} \MX\|_F^2 \leq  \epsilon d^2(\MA,\lspan{\MX}), \text{~~and~~}
0 \leq \|\MA \MX\|_F^2 - \|\tcd{\MA} \MX\|_F^2 \leq  \epsilon d^2(\MA,\lspan{\MX}).
\end{eqnarray*}
\end{lemma}

This means that the projections of $\tcd{\MA}$ and $\MA$ on any $\dimr$-dimensional subspace are close, when the projected dimension $t$ is sufficiently large compared to $\dimr$. Now, note that the difference between $\|\MP - \MP\MX\MX^\top\|_F^2$ and $\|\tcd{\MP} - \tcd{\MP} \MX \MX^\top\|_F^2$ is only related to $\|\MP \MX\|_F^2 - \|\tcd{\MP} \MX\|_F^2 = \sum_i [\|\MP_i \MX\|_F^2 - \|\tcd{\MP}_i \MX\|_F^2 ]$, each term in which is bounded by the lemma. So we can use $\tcd{\MP}$ as a proxy for $\MP$ in the PCA task. Again, computing PCA on $\tcd{\MP}$ is equivalent to computing SVD on $\MY$, as done in Algorithm {\sf disPCA}.  
These lead to the following theorem, which is implicit in~\citep{dan2013tiny}, stating that the algorithm 
can produce a $(1+\epsilon)$-approximation for the distributed PCA problem. 

\begin{theorem}\label{thm:lr_app}
Suppose Algorithm {\sf disPCA} 
takes parameters $t_1 \geq \dimr + \lceil 4\dimr/\epsilon \rceil -1$ and $t_2=\dimr$. Then
\begin{eqnarray*}
\|\MP - \MP\tc{\MV}{\dimr}(\tc{\MV}{\dimr})^\top\|_F^2 \leq (1+\epsilon) \min_{\MX} \|\MP - \MP\MX\MX^\top\|_F^2
\end{eqnarray*}
where the minimization is over $d\times \dimr$ orthonormal matrices $\MX$. The communication is $O(\frac{s\dimr d}{\epsilon})$ words.
\end{theorem}

\subsection{Guarantees for Distributed $\ell_2$-Error Fitting}\label{sec:l2}



Algorithm {\sf disPCA} 
can also be used as a pre-processing step for applications such as $\ell_2$-error fitting. In this section, we prove the correctness of Algorithm {\sf disPCA} 
as pre-processing for these applications. In particular, we show that by setting $t_1, t_2$ sufficiently large, the objective value of any solution merely changes when the original data $\MP$ is replaced the projected data $\gpca{\MP} =\MP\tc{\MV}{t_2}(\tc{\MV}{t_2})^\top$. Therefore,  the projected data serves as a proxy of the original data, i.e.\ , any distributed algorithm can be applied on the projected data to get a solution on the original data. As the dimension is lower, the communication cost is reduced. 
Formally,

\begin{theorem}\label{thm:coreset_gen}
Let $t_1=t_2 = O(\dimr k/ \epsilon^2)$ in Algorithm {\sf disPCA}   
for $\epsilon \in (0, 1/3)$, and let $\gpca{\MP} =\MP\tc{\MV}{t_2}(\tc{\MV}{t_2})^\top$. Then there exists a constant $c_0 \geq 0$ such that for any set of $k$ centers $\x$ in $\dimr$-Subspace $k$-Clustering,
$$
	(1-\epsilon) d^2(\MP,\x) \leq d^2(\gpca{\MP}, \x) + c_0 \leq (1+\epsilon) d^2(\MP,\x).
$$
Consequently,  any $\alpha$-approximate solution $\x$ on $\gpca{P}$ is a $\frac{1+\epsilon}{1-\epsilon}\alpha$-approximation on $\MP$. 
\end{theorem}

\noindent
\textbf{Remark 1} 
To see the consequence, let $\x^*$ denote the optimal solution. Then 
$$
	(1-\epsilon) d^2(\MP,\x) \leq d^2(\gpca{P}, \x) + c_0 \leq \alpha d^2(\gpca{P}, \x^*) + c_0 \leq \alpha (1+\epsilon)d^2(\MP, \x^*).
$$
Thus, the distributed PCA step only introduces a small multiplicative approximation factor of $(1+O(\epsilon))$.

\noindent
\textbf{Remark 2} 
The vectors $\tc{\MV}{t_2}$ are not the optimal $O(\dimr k/\epsilon^2)$ global principal components. They only satisfy a weaker property called close projection stated in Lemma~\ref{thm:DisPCA_P_l2}. On the other hand, as computing the first $\dimr$ global principal components (or equivalently rank-$\dimr$ approximation) itself is a special case of $\dimr$-Subspace $k$-Clustering (with $k=1$), Theorem~\ref{thm:coreset_gen} can also be applied. 
However, this will only lead to weaker guarantee than Theorem~\ref{thm:lr_app}.

More specifically, in this setting,  $\x$ is just one $\dimr$ subspace. Let $\MX$ be the othonormal matrix whose columns span $\x$, then $d^2(\MP,\x) = \|\MP - \MP\MX\MX^\top\|_F^2$. 
By Theorem~\ref{thm:coreset_gen}, the optimal $\dimr$ principal components for $\gpca{P}$,  which is just $\MV^{(r)}$ (the first $\dimr$ vectors in $\MV^{(t_2)}$), is a $(1+O(\epsilon))$-approximation for $\MP$:
$$
	 \|\MP - \MP\MV^{(r)}(\MV^{(r)})^\top\|_F^2 \leq \frac{1+\epsilon}{1-\epsilon} \min_{\MX}\|\MP - \MP\MX\MX^\top\|_F^2.
$$
Note that Theorem~\ref{thm:coreset_gen} requires setting $t_1 = O(r/\epsilon^2)$, which is larger than that in Theorem~\ref{thm:lr_app}.
Therefore, we focus on using Theorem~\ref{thm:coreset_gen} for other applications such as $k$-means.

\noindent
\textbf{Proof Sketch of Theorem~\ref{thm:coreset_gen}} 
The key to prove the theorem is the following close projection property of Algorithm {\sf disPCA} in Lemma~\ref{thm:DisPCA_P_l2}. Intuitively, it means that for any low dimensional subspace spanned by $\MX$,  the projections of $\MP$ and $\gpca{\MP}$ on the subspace are close.
To prove Theorem~\ref{thm:coreset_gen} by this,  we choose $\MX$ to be the orthonormal basis of the subspace spanning the centers. Since the problem only involves $l_2$ error, the difference between the objective values of $\MP$ and $\gpca{\MP}$ can be decomposed into two terms depending only on $ \|\MP \MX - \gpca{\MP} \MX\|_F^2$ and $\|\MP \MX \|_F^2 - \| \gpca{\MP} \MX\|_F^2$ respectively, which are small as shown by the lemma. The complete proof of Theorem~\ref{thm:coreset_gen} is provided in Appendix~\ref{app:disPCA_coreset}.

\begin{lemma}\label{thm:DisPCA_P_l2}
Let $t_1=t_2 =O(k/\epsilon)$ in Algorithm {\sf disPCA}.     
Then for any $d \times k$ matrix $\MX$ with orthonormal columns,
$
0 \leq \|\MP \MX - \gpca{\MP} \MX\|_F^2   \leq \epsilon d^2(\MP, \lspan{\MX})$, and
$0 \leq  \|\MP \MX\|_F^2 - \|\gpca{\MP} \MX\|_F^2  \leq \epsilon d^2(\MP, \lspan{\MX}).
$
\end{lemma}

\begin{proofsketch}
We first introduce some auxiliary variables for the analysis, which act as intermediate connections between $\MP$ and $\gpca{\MP}$.
Imagine we perform two kinds of projections: first project $\MP_i$ to $\tcd{\MP}_i = \MP_i \tc{\MV_i}{t_1} (\tc{\MV_i}{t_1})^\top$,
then project $\tcd{\MP}_i$ to $\overline{\MP}_i= \tcd{\MP}_i \tc{\MV}{t_2} (\tc{\MV}{t_2})^\top$.
Let $\tcd{\MP}$ denote the vertical concatenation of $\tcd{\MP}_i$ and let $\overline{\MP}$ denote the vertical concatenation of $\overline{\MP}_i$. 
These variables are designed so that the difference between $\MP$ and $\tcd{\MP}$ and that between $\tcd{\MP}$ and $\overline{\MP}$ are easily bounded. 

Our proof then proceeds by first bounding these differences, and then bounding that between $\overline{\MP}$ and $\gpca{\MP}$.
Take the second statement as an example. We have the following decomposition:
\begin{eqnarray*}
\|\MP \MX\|_F^2 - \|\gpca{\MP} \MX\|_F^2 
&= & 
\left[ \|\MP \MX\|_F^2 - \|\tcd{\MP} \MX\|_F^2 \right]+ \left[ \|\tcd{\MP} \MX\|_F^2 - \|\overline{\MP} \MX\|_F^2 \right] + \left[ \|\overline{\MP} \MX\|_F^2 - \|\gpca{\MP} \MX\|_F^2 \right].
\end{eqnarray*}
The first term is just $\sum_{i=1}^s \left[\|\MP_i \MX\|_F^2 - \|\tcd{\MP}_i \MX\|_F^2 \right]$, each of which can be bounded by Lemma~\ref{lem:pca_app}, since $\tcd{\MP}_i$ is the SVD truncation of $\MP$. The second term can be bounded similarly.
The more difficult part is the third term. Note that $\overline{\MP}_i = \tcd{\MP}_i \MZ$, $\gpca{\MP}_i = \MP_i \MZ$ where $\MZ :=\tc{\MV}{t_2}(\tc{\MV}{t_2})^\top \MX$, leading to
\begin{eqnarray*}
	\|\overline{\MP} \MX\|_F^2 - \|\gpca{\MP} \MX\|_F^2 = \sum_{i=1}^s \left[ \|  \tcd{\MP}_i \MZ \|_F^2 -  \| \MP_i \MZ \|_F^2 \right].
\end{eqnarray*}
Although $\MZ$ is not orthonormal as required by Lemma~\ref{lem:pca_app}, we prove a generalization (Lemma~\ref{lem:pca_app_gen} in the appendix) which can be applied to show that the third term is indeed small.

The bound on $ \|\MP \MX - \gpca{\MP} \MX\|_F^2$ can be proved by a similar argument. See Appendix~\ref{app:disPCA_closeProj} for details.
\end{proofsketch}

\littleheader{Application to $k$-Means Clustering}
To see the implication of Theorem~\ref{thm:coreset_gen}, consider the $k$-means clustering problem.
We can first perform any other possible dimension reduction to dimension $d'$ so that the $k$-means cost is preserved up to accuracy $\epsilon$,  and then run Algorithm {\sf disPCA} 
and finally run any distributed $k$-means clustering algorithm on the projected data
to get a good approximate solution.  
For example, in the first step we can set $d' = O(\log n / \epsilon^2)$  using a Johnson-Lindenstrauss transform,
or we can perform no reduction and simply use the original data.

As a concrete example, we can use original data ($d'=d$),
then run Algorithm {\sf disPCA}, 
and finally run the distributed clustering algorithm in~\citep{disClustering13}
which uses any non-distributed $\alpha$-approximation algorithm as a subroutine
and computes a $(1+\epsilon)\alpha$-approximate solution.
The resulting algorithm is presented in Algorithm~\ref{alg:disKmeans}.

\begin{algorithm}[t]
\caption{Distributed $k$-means clustering}
\label{alg:disKmeans}
\begin{algorithmic}[1]
\REQUIRE{$\{\MP_i\}_{i=1}^s$, $k \in \N_+$ and $\epsilon \in (0,1/3)$, a non-distributed $\alpha$-approximation algorithm $\mathcal{A}_\alpha$}
\STATE{Run Algorithm {\sf disPCA} 
with $t_1=t_2=O(k/\epsilon^2)$ to get $\ME=\MV^{(t_2)}$, and send $\ME$ to all nodes.}
\STATE{Run the distributed $k$-means clustering algorithm in~\citep{disClustering13} on $\{\MP_i\ME\ME^\top\}_{i=1}^s$, using $\mathcal{A}_\alpha$ as a subroutine, to get $k$ centers $\x$.}
\ENSURE{$\x$.}
\end{algorithmic}
\end{algorithm}

\begin{theorem}\label{cor:clustering}
With probability at least $1-\delta$,
Algorithm~\ref{alg:disKmeans} outputs a $(1+\epsilon)^2\alpha$-approximate solution for distributed $k$-means clustering.
The total communication cost of Algorithm~\ref{alg:disKmeans} is $O(\frac{sk}{\epsilon^2})$ vectors in $\R^d$
plus $O\left(\frac{1}{\epsilon^4}(\frac{k^2}{\epsilon^2} + \log\frac{1}{\delta}) + sk\log\frac{sk}{\delta}\right)$ vectors in $\R^{O(k/\epsilon^2)}$.
\end{theorem}


\section{Fast Distributed PCA}\label{sec:fastl2}

\littleheader{Subspace Embeddings}
One can significantly improve the time of the distributed PCA algorithms by using subspace embeddings,
while keeping similar guarantees as in Lemma~\ref{thm:DisPCA_P_l2}, which suffice for $l_2$-error fitting.
More precisely, a subspace embedding matrix $\MH \in \R^{\ell\times n}$ for a matrix $\MA \in \R^{n\times d}$ has the property that
for all vectors $y \in \R^d$,
$\|\MH \MA y\|_2 = (1 \pm \epsilon) \|\MA y\|_2.$
Suppose independently, each node $v_i$ chooses a random subspace embedding matrix $\MH_i$ for its local data $\MP_i$.
Then, they run Algorithm {\sf disPCA} 
on the embedded data $\{\MH_i\MP_i\}_{i=1}^s$ instead of
on the original data $\{\MP_i\}_{i=1}^s$.
%

The work of \citep{s06} pioneered subspace embeddings. 
The recent fast sparse subspace
embeddings~\citep{clarkson2013low} and its optimizations~\citep{meng2013low,nelson2012osnap}
are particularly suitable for large scale sparse data sets, since
their running time is linear in the number of non-zero entries
in the data matrix, and they also preserve the sparsity of the data.
The algorithm takes as input an $n\times d$ matrix $\MA$ and a parameter $\ell$,
and outputs an $\ell \times d$ embedded matrix $\MA'=\MH\MA$
(the embedded matrix $\MH$ does need to be built explicitly).
The embedded matrix is constructed as follows: initialize $\MA'=\mathbf{0}$;
for each row in $\MA$, multiply it by $+1$ or $-1$ with equal probability,
then add it to a row in $\MA'$ chosen uniformly at random.

The success probability is constant, while we need to set it to be $1-\delta$ where $\delta=\Theta(1/s)$.
Known results which preserve the number of non-zero entries of $\MH$ to be $1$ per column increase the dimension of $\MH$
by a factor of $s$. To avoid this, 
we propose an approach to boost the success probability by computing $O(\log \frac{1}{\delta})$
independent embeddings, each with only constant success probability,
and then run a cross validation style procedure to find one which succeeds with probability $1-\delta$.
More precisely, we compute the SVD of all embedded matrices $\MH_j\MA = \MU_j \MD_j \MV_j^\top$,
and find a $j \in [r]$ such that for at least half of the indices $j' \neq j$,
all singular values of $\MD_j \MV_j^\top \MV_{j'} \MD_{j'}^\top$ are in $[1 \pm O(\epsilon)]$ (see Algorithm~\ref{alg:success} in the appendix). 
The reason why such an embedding $\MH_j \MA$ succeeds with high probability is as follows.
Any two successful embeddings $\MH_j \MA$ and $\MH_{j'} \MA$, by definition, satisfy that
$\| \MH_j \MA x \|^2_2  = (1 \pm O(\epsilon)) \| \MH_{j'} \MA x \|_2^2$ for all $x$,
which we show is equivalent to passing the test on the singular values.
Since with probability at least $1-\delta$, $9/10$ fraction of the embeddings are successful,
it follows that the one we choose is successful with probability $1-\delta$.

\littleheader{Randomized SVD}
The exact SVD of an $n \times d$ matrix is impractical in the case
when $n$ or $d$ is large.
Here we show that the randomized SVD algorithm from~\citep{halko2011finding} can be applied to speed up the computation
without compromising the quality of the solution much. We need to use their specific form of randomized SVD since the error
is with respect to the spectral norm, rather than the Frobenius norm, and so can be much smaller as needed by our applications.

The algorithm first probes the row space of the $\ell \times d$ input matrix $\MA$ with an $\ell \times 2t$ random matrix $\MOmega$ and 
orthogonalizes the image of $\MOmega$ to get a basis $\MQ$ (i.e.,\ QR-factorize $\MA^\top \MOmega$);
projects the data to this basis and computes the SVD factorization on the smaller matrix $\MA\MQ$. It also performs $q$ power iterations to push the basis towards the top $t$ singular vectors.

\littleheader{Fast Distributed PCA  for $l_2$-Error Fitting}
We modify Algorithm {\sf disPCA}
by first having each node do a subspace embedding locally, then replace each SVD invocation with a randomized SVD invocation.
We thus arrive at Algorithm~\ref{alg:fastDisPCA}.
For $\ell_2$-error fitting problems, by combining approximation guarantees of the randomized techniques with that of distributed PCA, we are
able to prove:
\begin{theorem}\label{thm:fastDisPCA_coreset} 
Suppose Algorithm~\ref{alg:fastDisPCA} takes $\epsilon \in (0,1/2]$, $t_1=t_2 = O(\max\cbr{ \frac{k}{\epsilon^2}, \log \frac{s}{\delta}}), \ell = {O}(\frac{d^2}{\epsilon^2}), q=O(\max\{\log \frac{d}{\epsilon},\log \frac{sk}{\epsilon}\})$ as input, and sets the failure probability of each local subspace embedding to $\delta' = \delta/2s$.
Let $\gpca{P}=\MP\MV \MV^\top$. 
Then with probability at least $1-\delta$,
there exists a constant $c_0 \geq 0$, such that
for any set of $k$ points $\x$,
\begin{eqnarray*} (1-\epsilon) d^2(\MP,\x) - \epsilon \|\MP\MX\|_F^2 \leq  d^2(\gpca{P}, \x) + c_0 \leq  (1+\epsilon) d^2(\MP,\x) + \epsilon \|\MP\MX\|_F^2
\end{eqnarray*}
where $\MX$ is an orthonormal matrix whose columns span $\x$.
The total communication is $O(skd/\epsilon^2)$ and the total time is
$O\left(\nnz(\MP) + s\left[\frac{d^3k}{\epsilon^4} + \frac{k^2 d^2}{\epsilon^6}\right] \log \frac{d}{\epsilon}\log\frac{sk}{\delta\epsilon}\right)$.
\end{theorem}

\begin{proofsketch}
It suffices to show that  $\gpca{\MP}$ enjoys the close projection property as in Lemma~\ref{thm:DisPCA_P_l2}, i.e., $\| \MP \MX - \gpca{\MP} \MX \|^2_F \approx 0$ and $\| \MP \MX\|^2_F - \|\gpca{\MP} \MX \|^2_F \approx 0$ for any orthonormal matrix $\MX$ whose columns span a low dimensional subspace. 
Note that Algorithm~\ref{alg:fastDisPCA} is just running Algorithm {\sf disPCA} (with randomized SVD) on $\MT\MP$ where $\MT= \text{diag}(\MH_1, \MH_2, \dots, \MH_s)$, so we first show that $\MT\gpca{\MP}$ enjoys this property. 
But now exact SVD is replaced with randomized SVD, for which we need to use the spectral error bound to argue that the error introduced is small. More precisely, for a matrix $\MA$ and its SVD truncation $\tcd{\MA}$ computed by randomized SVD, it is guaranteed that the spectral norm of $\MA - \tcd{\MA}$ is small, then $\|(\MA - \tcd{\MA}) \MX\|_F$ is small for any $\MX$ with small Frobenius norm, in particular, the orthonormal basis spanning a low dimensional subspace. This then suffices to guarantee $\MT\gpca{\MP}$ enjoys the close projection property.
Given this, it suffices to show that $\gpca{\MP}$ enjoys this property as $\MT\gpca{\MP}$, which follows from the definition of a subspace embedding.
\end{proofsketch}

\begin{algorithm}[t]
\caption{Fast Distributed PCA for $l_2$-Error Fitting}
\label{alg:fastDisPCA}
\begin{algorithmic}[1]
\REQUIRE{$\{\MP_i\}_{i=1}^s$; parameters $ t_1,t_2$ for Algorithm {\sf disPCA}; $\ell, q$ for randomized techniques.}
\FOR{each node $v_i \in \Vcal$ }
    \STATE{Compute subspace embedding $\MP'_i=\MH_i \MP_i$.}
\ENDFOR
\STATE{Run Algorithm {\sf disPCA} on $\{\MP'_i\}_{i=1}^s$ to get $\MV$, where the SVD is randomized.}
\STATE{\quad $\mybullet$ dimension: $[\MP_i]_{n_i \times d}, [\MP'_i]_{\ell \times d}, [\MV]_{d\times 2t_2}$}
\ENSURE{$\MV$.}
\end{algorithmic}
\end{algorithm}

%
%

\section{Experiments}

Our focus is to show the randomized techniques used in Algorithm~\ref{alg:fastDisPCA} reduce the time taken significantly without
compromising the quality of the solution. We perform experiments for three tasks: rank-$r$ approximation,  $k$-means clustering and principal component regression (PCR).



\littleheader{Datasets} We choose the following real world datasets from UCI repository~\citep{Bache+Lichman:2013} for our
experiments. For low rank approximation and $k$-means clustering,  we choose two medium size datasets NewsGroups ($18774 \times 61188$) and MNIST 
($70000 \times 784$), and two large-scale Bag-of-Words datasets: NYTimes news articles  (BOWnytimes) ($300000  \times 102660$)  and PubMed abstracts (BOWpubmed) ($8200000 \times 141043$).
We use $r=10$ for rank-$r$ approximation and $k = 10$ for $k$-means clustering.
For PCR, we use MNIST and further choose YearPredictionMSD ($515345  \times 90$), CTslices ($53500 \times 386$),
and a large dataset MNIST8m ($800000 \times 784$).

\littleheader{Experimental Methodology}
The algorithms are evaluated on a star network. The number of nodes is
$s=25$ for medium-size datasets, and $s=100$ for the larger ones.
We distribute the data over the nodes using a weighted partition,
where each point is distributed to the nodes with probability proportional to the node's weight chosen from the power law with
parameter $\alpha=2$.

For each projection dimension, we first construct the projected data using distributed PCA.
For low rank approximation, we report the ratio between the cost of the obtained solution to that of the solution computed by SVD on the global data.
For $k$-means, we run the algorithm in~\citep{disClustering13} (with Lloyd's method as a subroutine) on the projected data to get a solution. Then we report the ratio between the cost of the above solution to that of a solution obtained by running Lloyd's method directly on the global data.
For PCR, we perform regression on the projected data to get a solution. Then we report the ratio between the error of the above solution to that of a solution obtained by PCR directly on the global data. We stop the algorihtm if it takes more than 24 hours.
For each projection dimension and each algorithm with randomness, the average ratio over 5 runs is reported.

\littleheader{Results}
 Figure~\ref{fig:lowrank_cost} shows the results for low rank approximation. We observe that the error of the fast distributed PCA is comparable to that of the exact solution computed directly on the global data. This is also observed for distributed PCA with one or none of subspace embedding and randomized SVD. 
Furthermore, the error of the fast PCA is comparable to that of normal PCA, which means that the speedup techniques merely affects the accuracy of the solution. The second row shows the computational time, which suggests a significant decrease in the time taken to run the fast distributed PCA. For example, on NewsGroups, the time of the fast distributed PCA improves over that of normal distributed PCA by a factor between $10$ to $100$. On the large dataset BOWpubmed, the normal PCA takes too long to finish and no results are presented, while the speedup versions produce good results in reasonable time. The use of the randomized techniques gives us a good performance improvement while keeping the solution quality almost the same.

Figure~\ref{fig:kmeans_cost} and Figure~\ref{fig:pcr_cost} show the results for $k$-means clustering and PCR respectively.
Similar to that for low rank approximation, we observe that the distributed solutions are almost as good as that computed directly on the global data, and the speedup merely affects the solution quality. We again observe a huge decrease in the running time by the speedup techniques.



\newcommand{\betweenWidth}{.0in}
\newcommand{\sfigWidth}{0.2\textwidth}
\newcommand{\sfigHeight}{0.15\textwidth}
\newcommand{\ssfigHeight}{0.14\textwidth}

\begin{figure*}[p]
\vspace{-0.3in}
\begin{center}
    \subfigure[NewsGroups]{\includegraphics[width=\sfigWidth,height=\ssfigHeight]{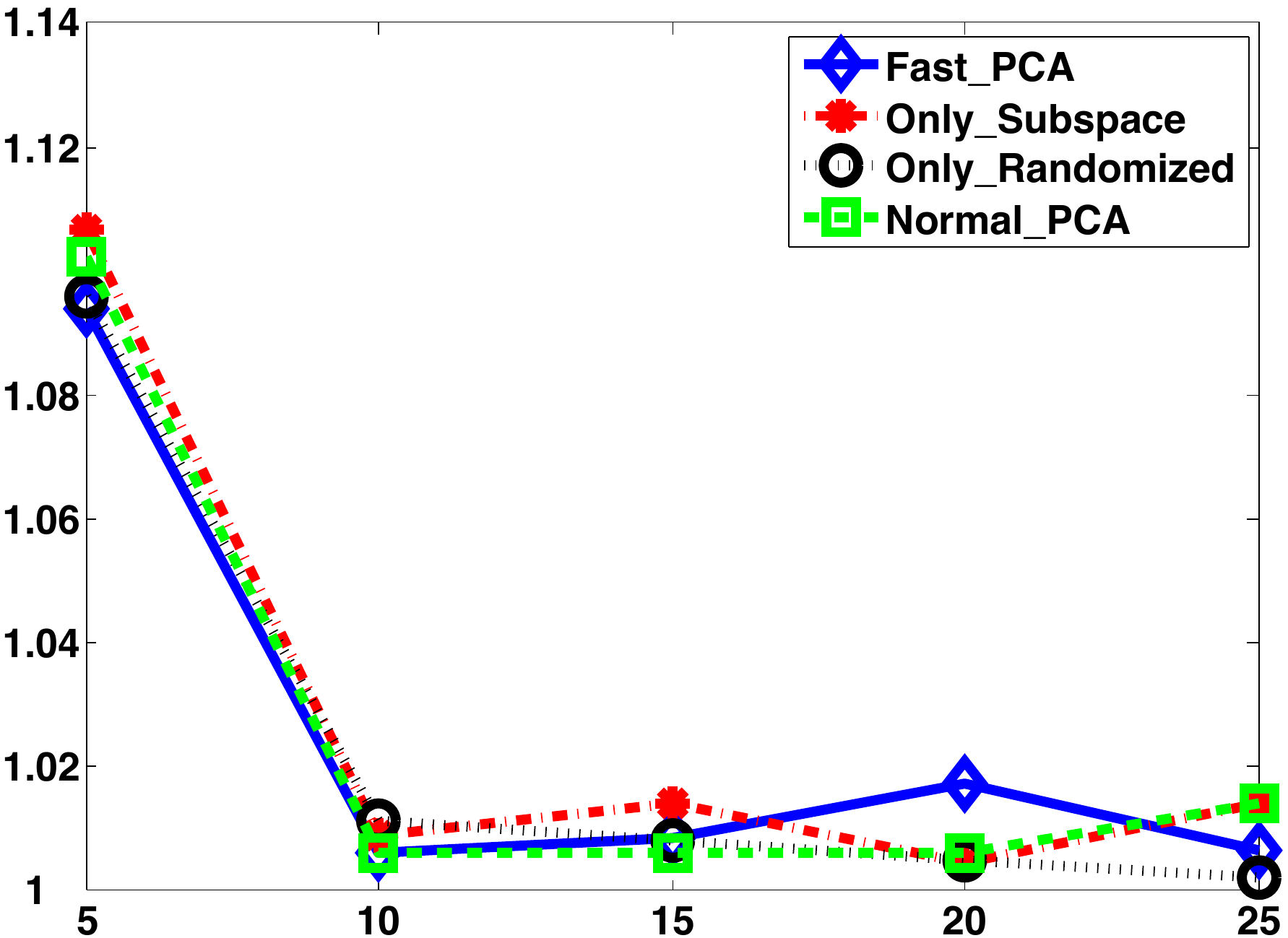}}
    \hspace*{\betweenWidth}
    \subfigure[MNIST]{\includegraphics[width=\sfigWidth,height=\ssfigHeight]{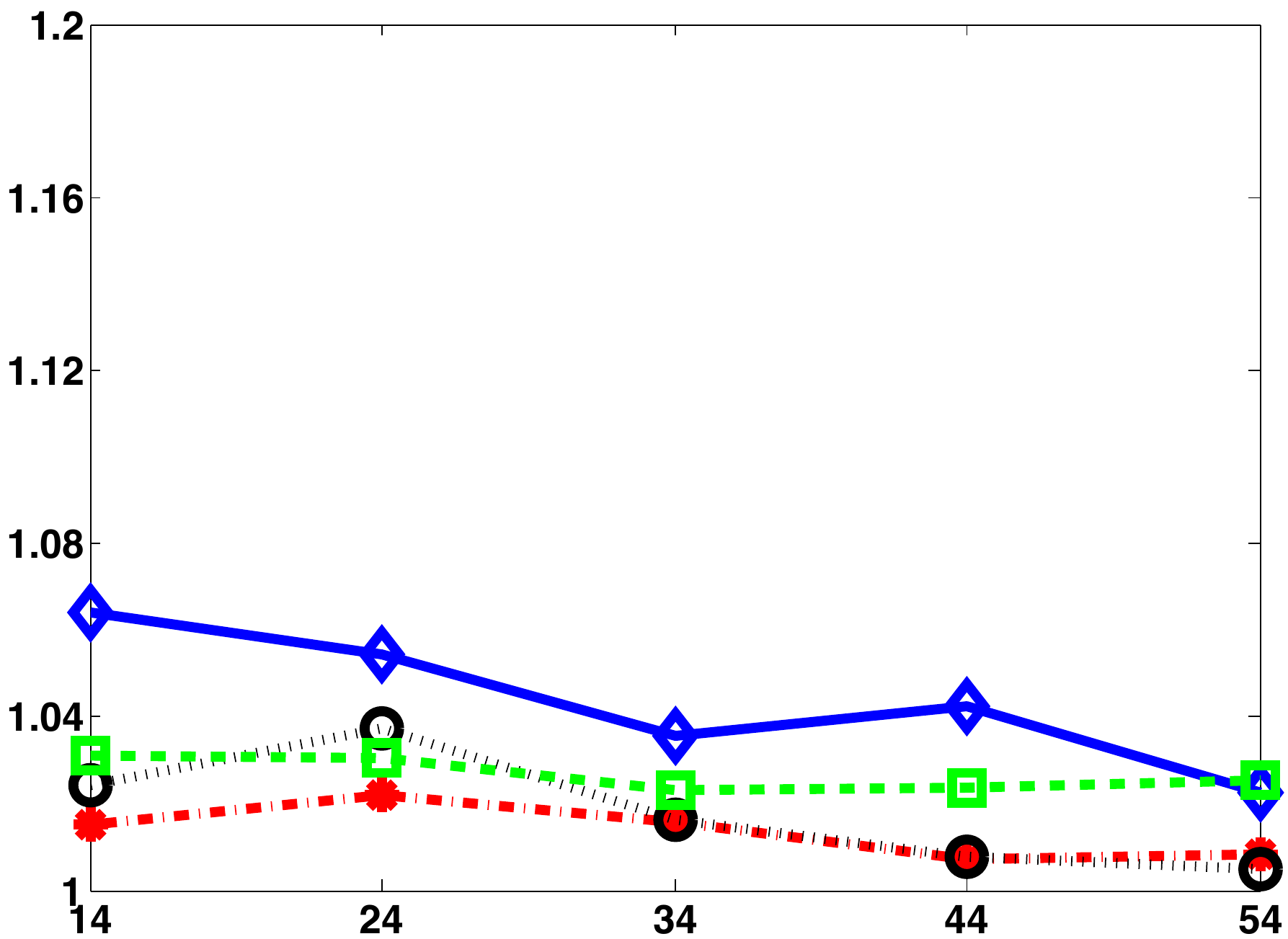}}
    \hspace*{\betweenWidth}
    \subfigure[BOWnytimes]{\includegraphics[width=\sfigWidth,height=\ssfigHeight]{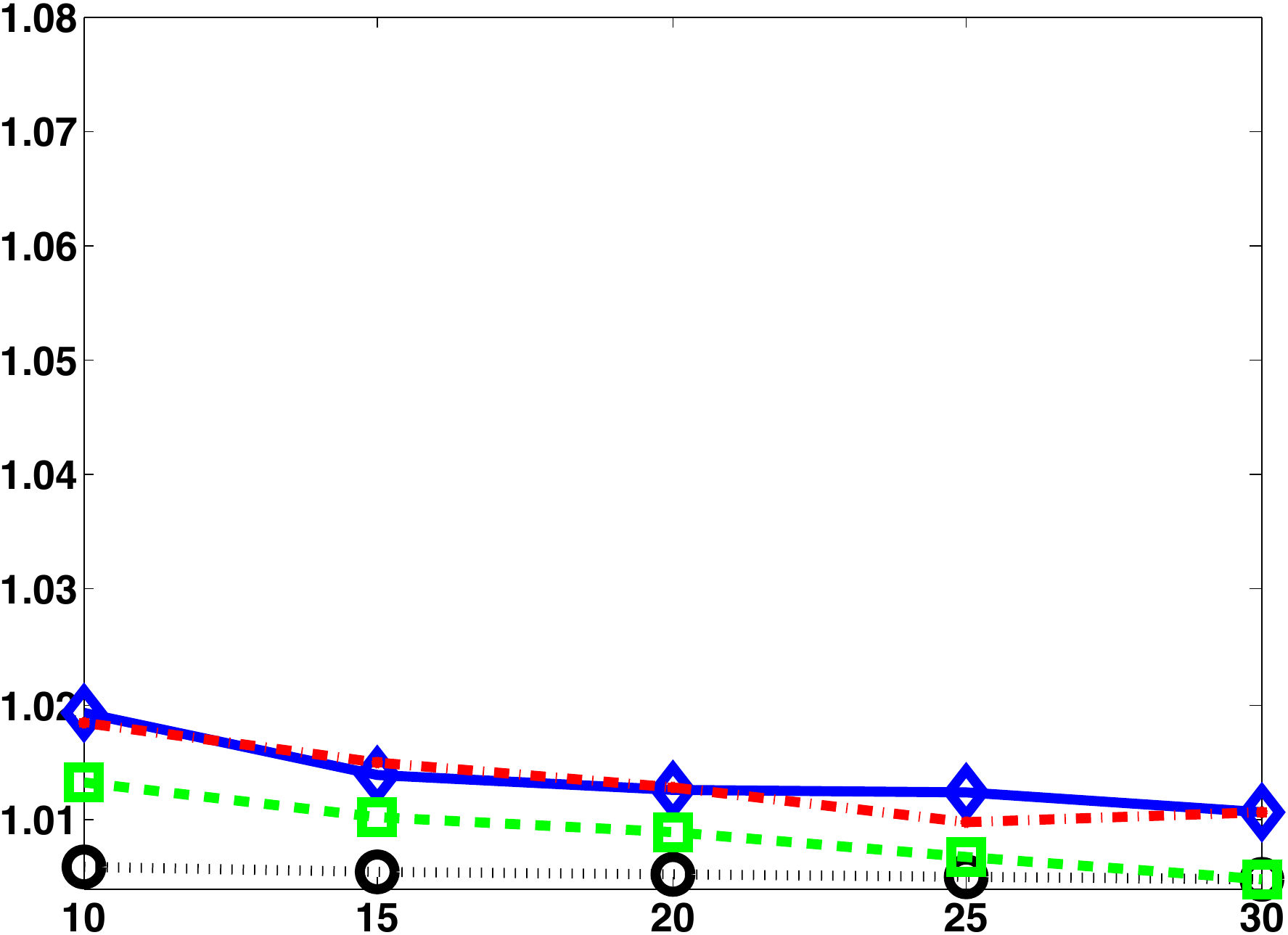}}
    \hspace*{\betweenWidth}
    \subfigure[BOWpubmed]{\includegraphics[width=\sfigWidth,height=\ssfigHeight]{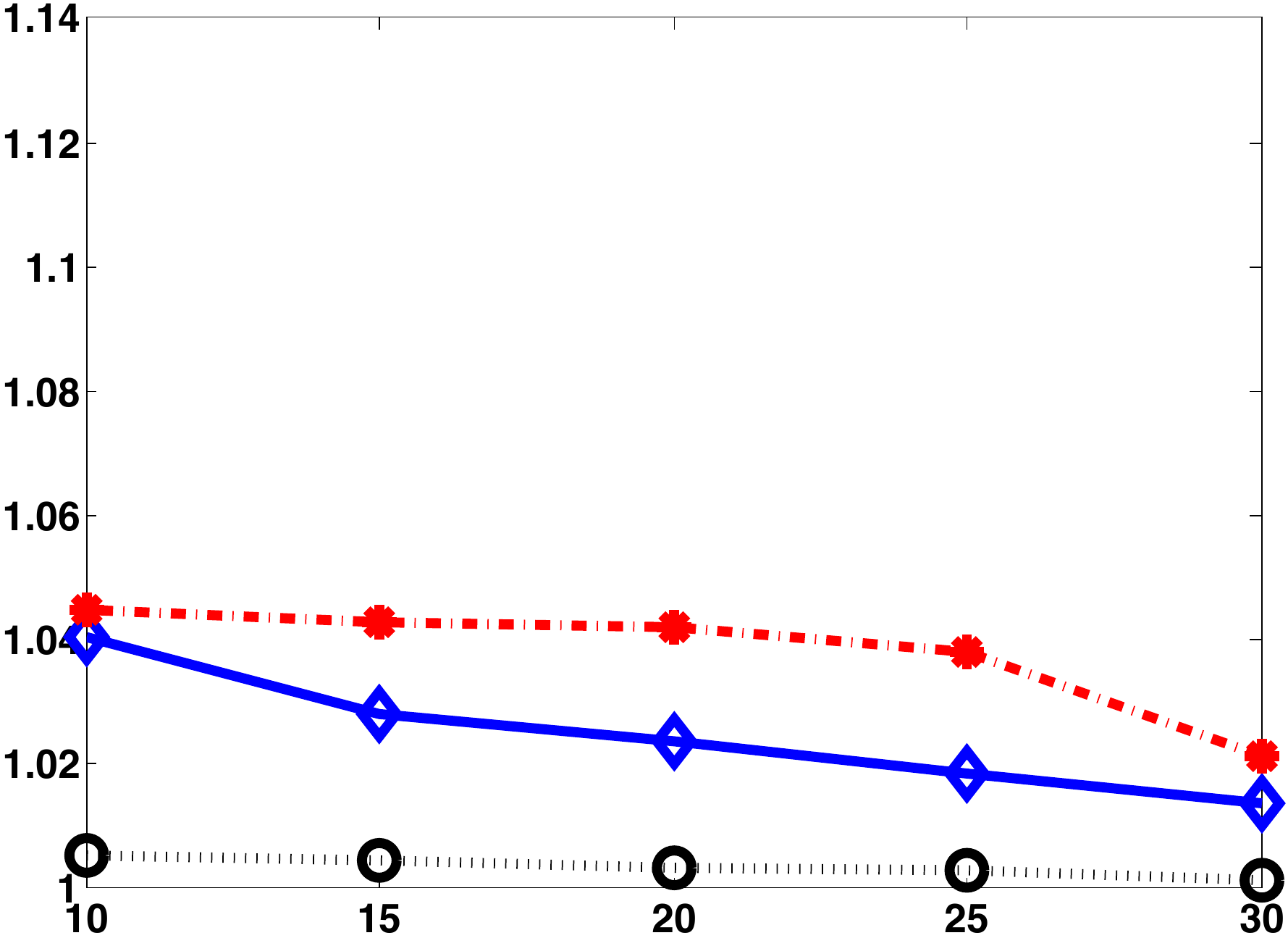}}\\
    \subfigure[NewsGroups]{\includegraphics[width=\sfigWidth,height=\sfigHeight]{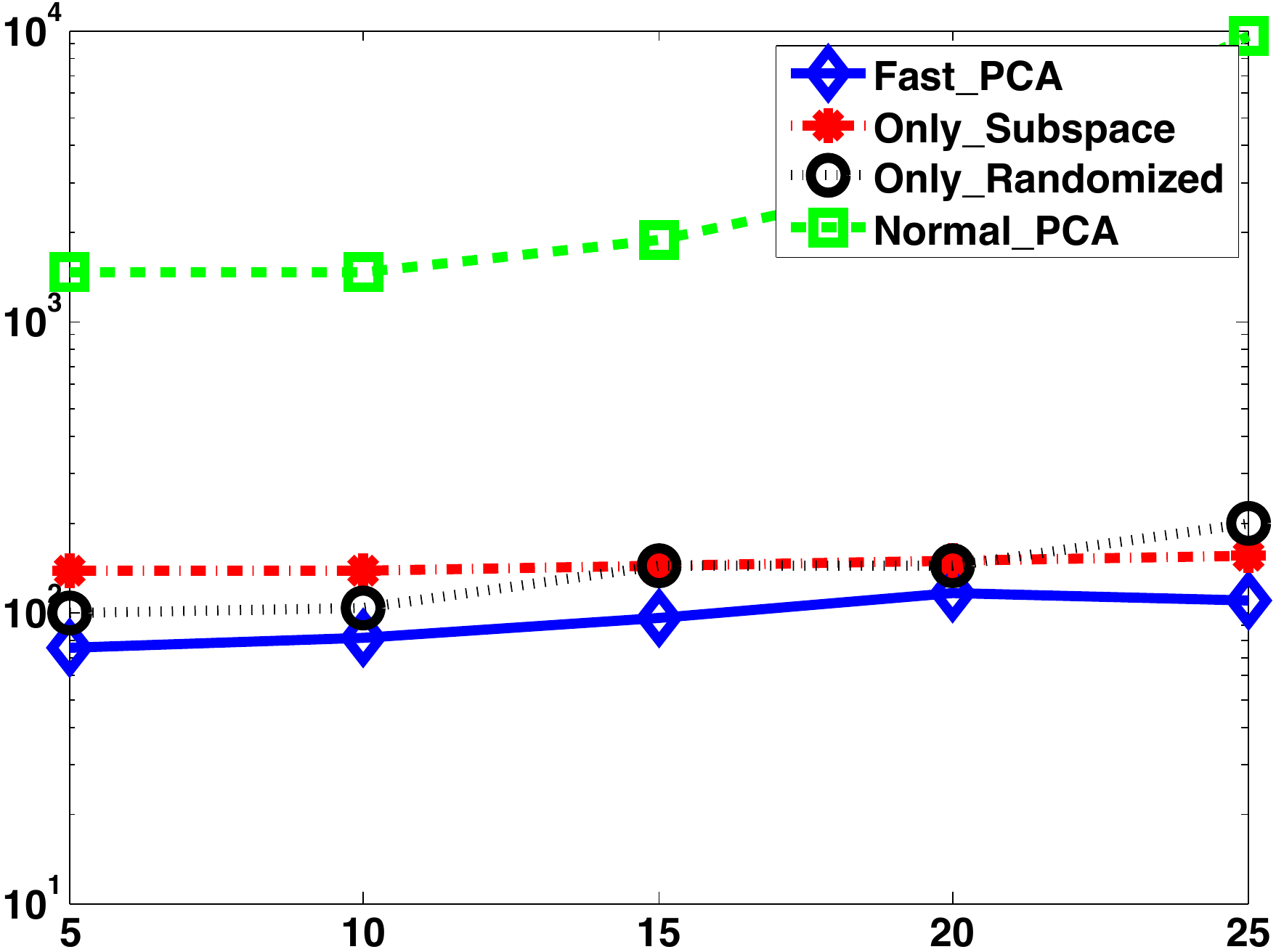}}
    \hspace*{\betweenWidth}
    \subfigure[MNIST]{\includegraphics[width=\sfigWidth,height=\sfigHeight]{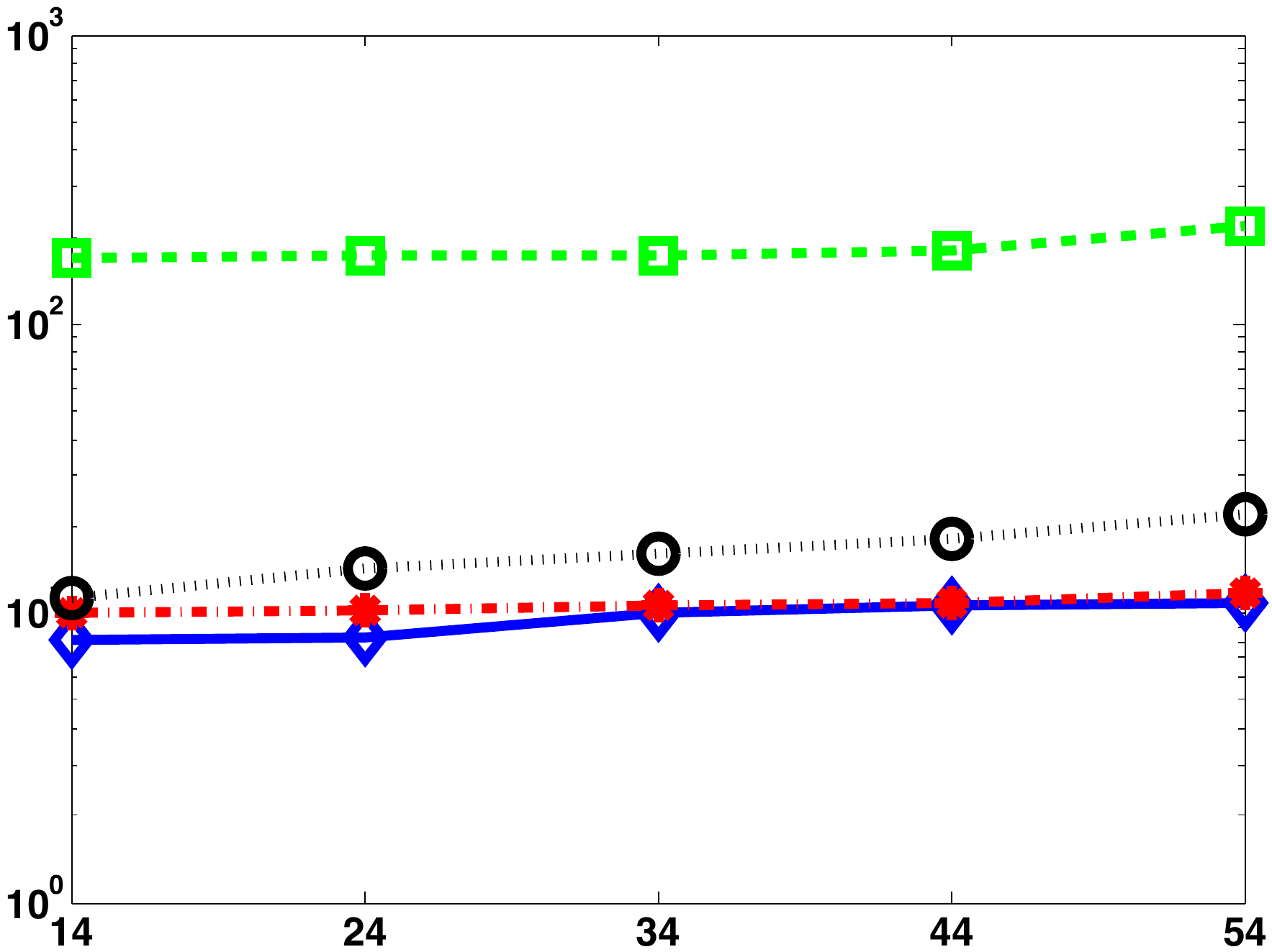}}
    \hspace*{\betweenWidth}
    \subfigure[BOWnytimes]{\includegraphics[width=\sfigWidth,height=\sfigHeight]{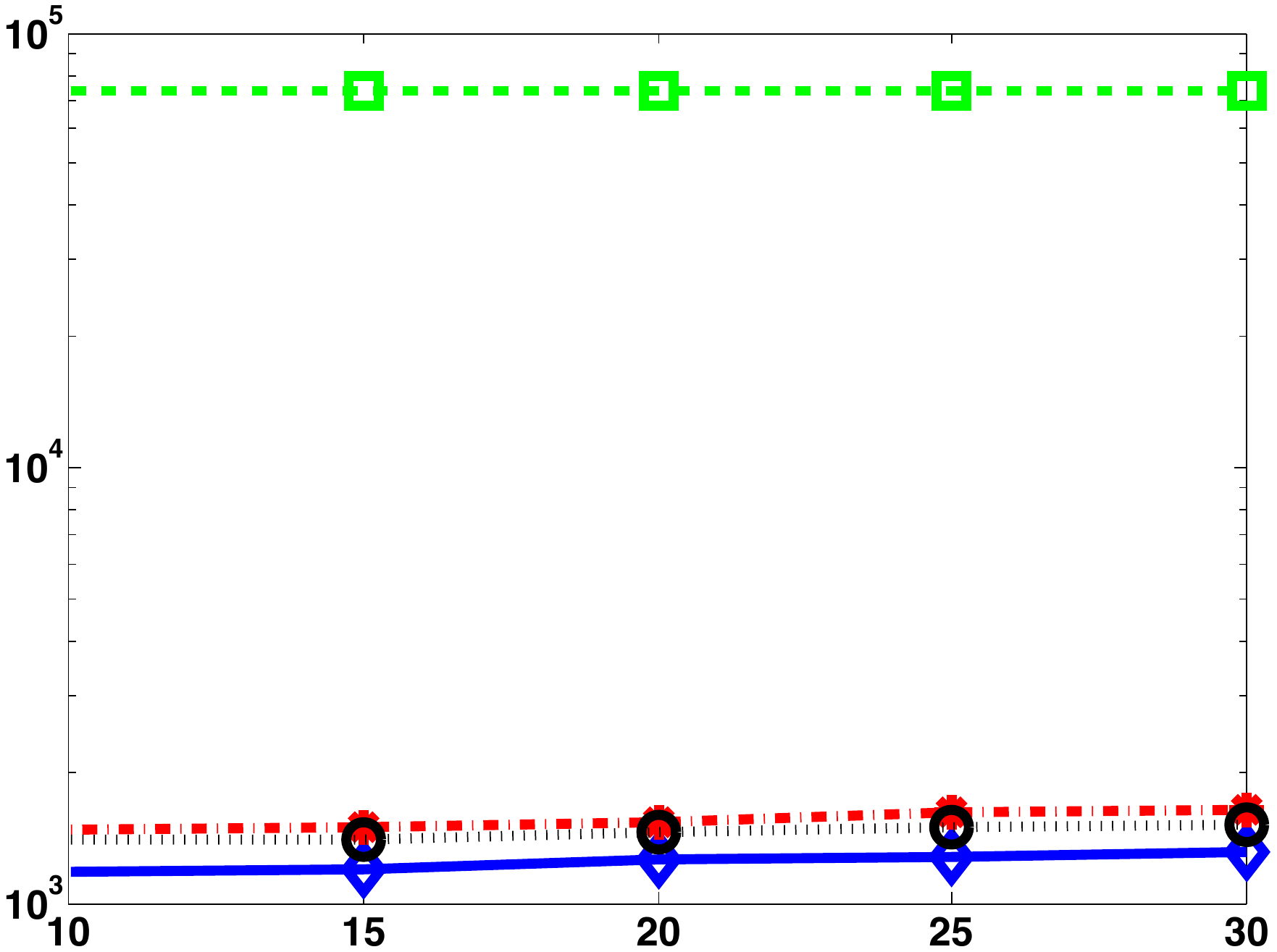}}
    \hspace*{\betweenWidth}
    \subfigure[BOWpubmed]{\includegraphics[width=\sfigWidth,height=\sfigHeight]{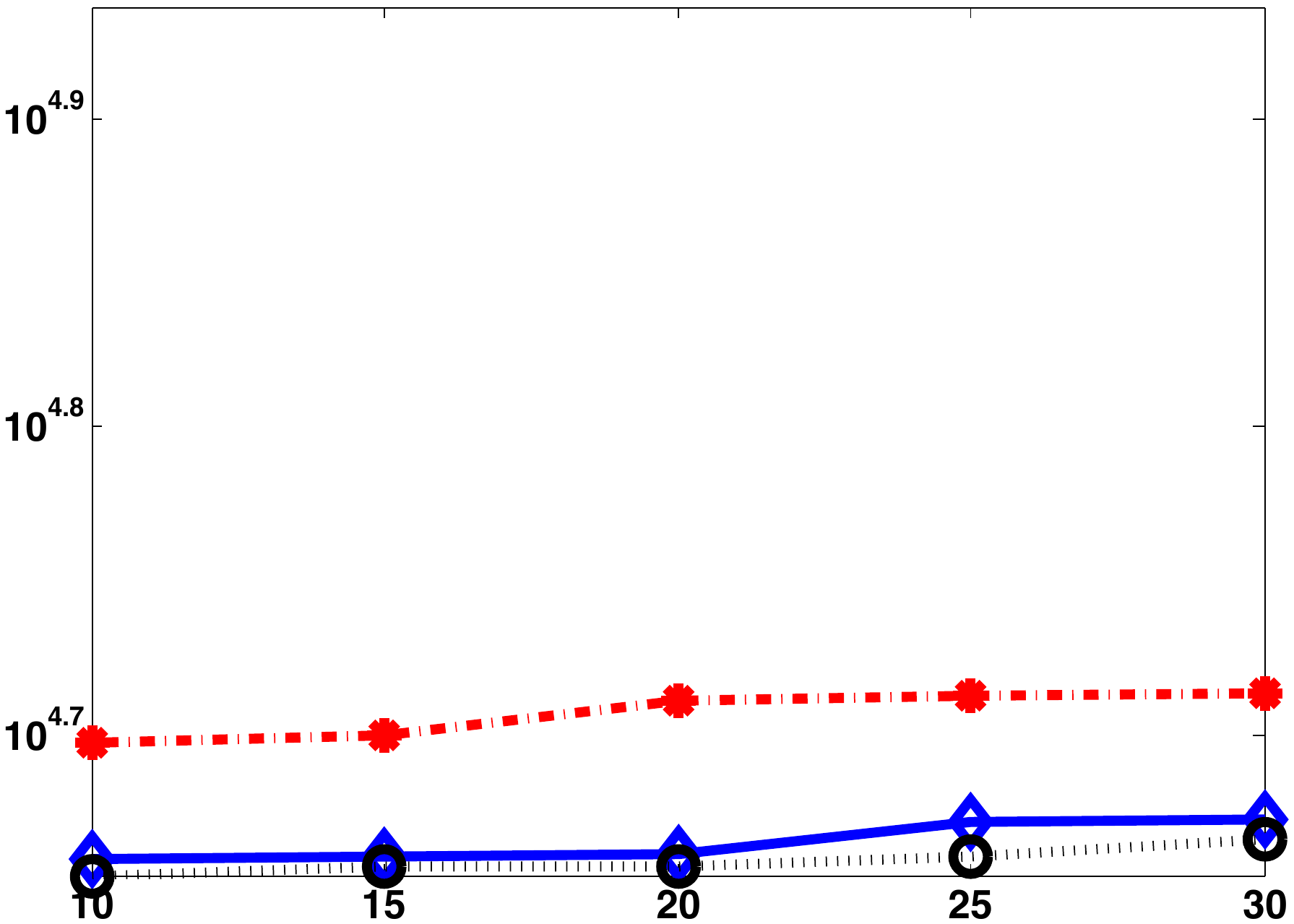}}
\end{center}
\vspace{-0.3in}
\caption{
{
Low rank approximation. First row: error (normalized by baseline) v.s. projection dimension. 
Second row: time v.s. projection dimension.
}\label{fig:lowrank_cost}
}
\vspace{-0.1in}
\begin{center}
    \subfigure[NewsGroups]{\includegraphics[width=\sfigWidth,height=\sfigHeight]{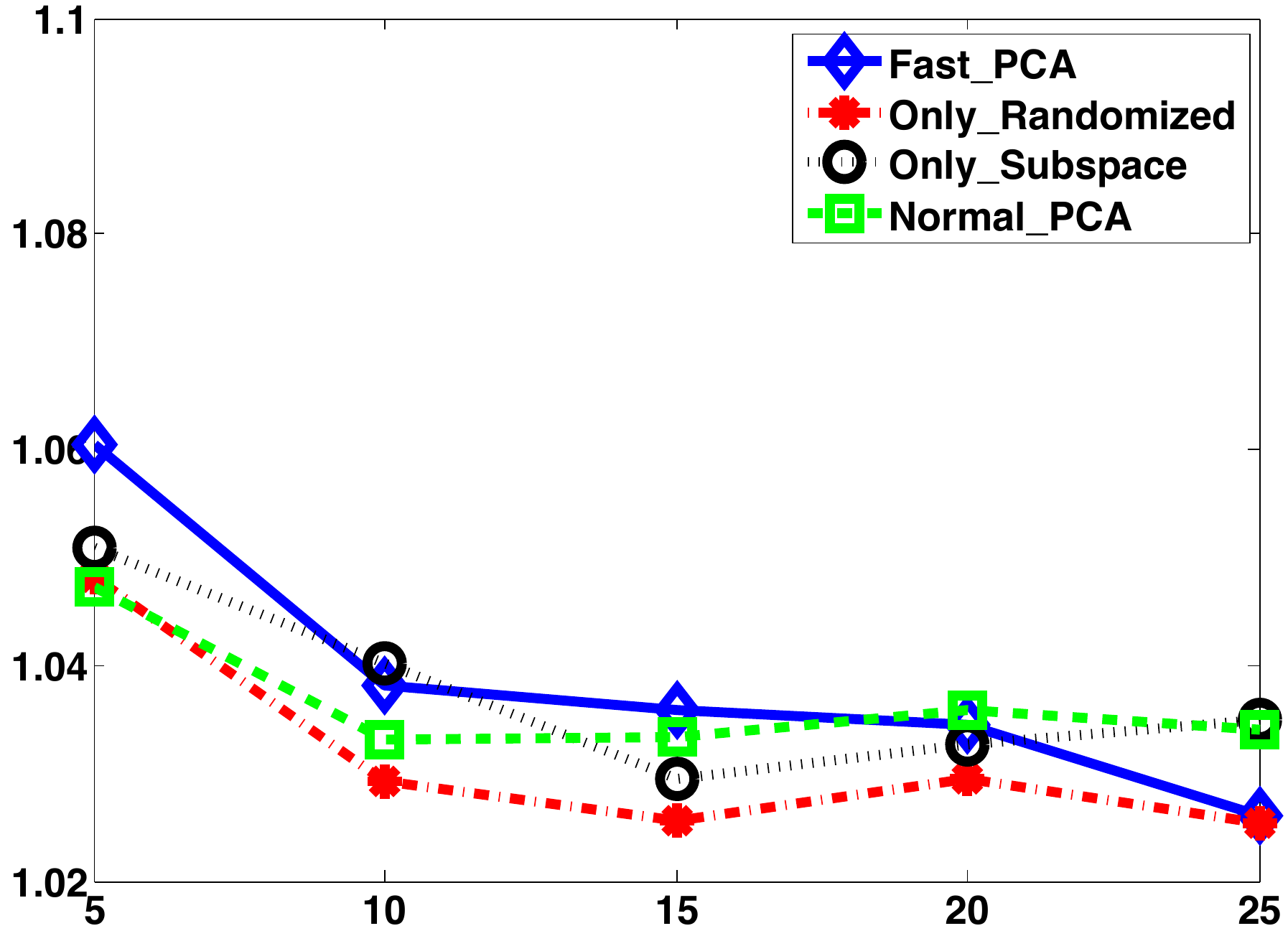}}
    \hspace*{\betweenWidth}
    \subfigure[MNIST]{\includegraphics[width=\sfigWidth,height=\sfigHeight]{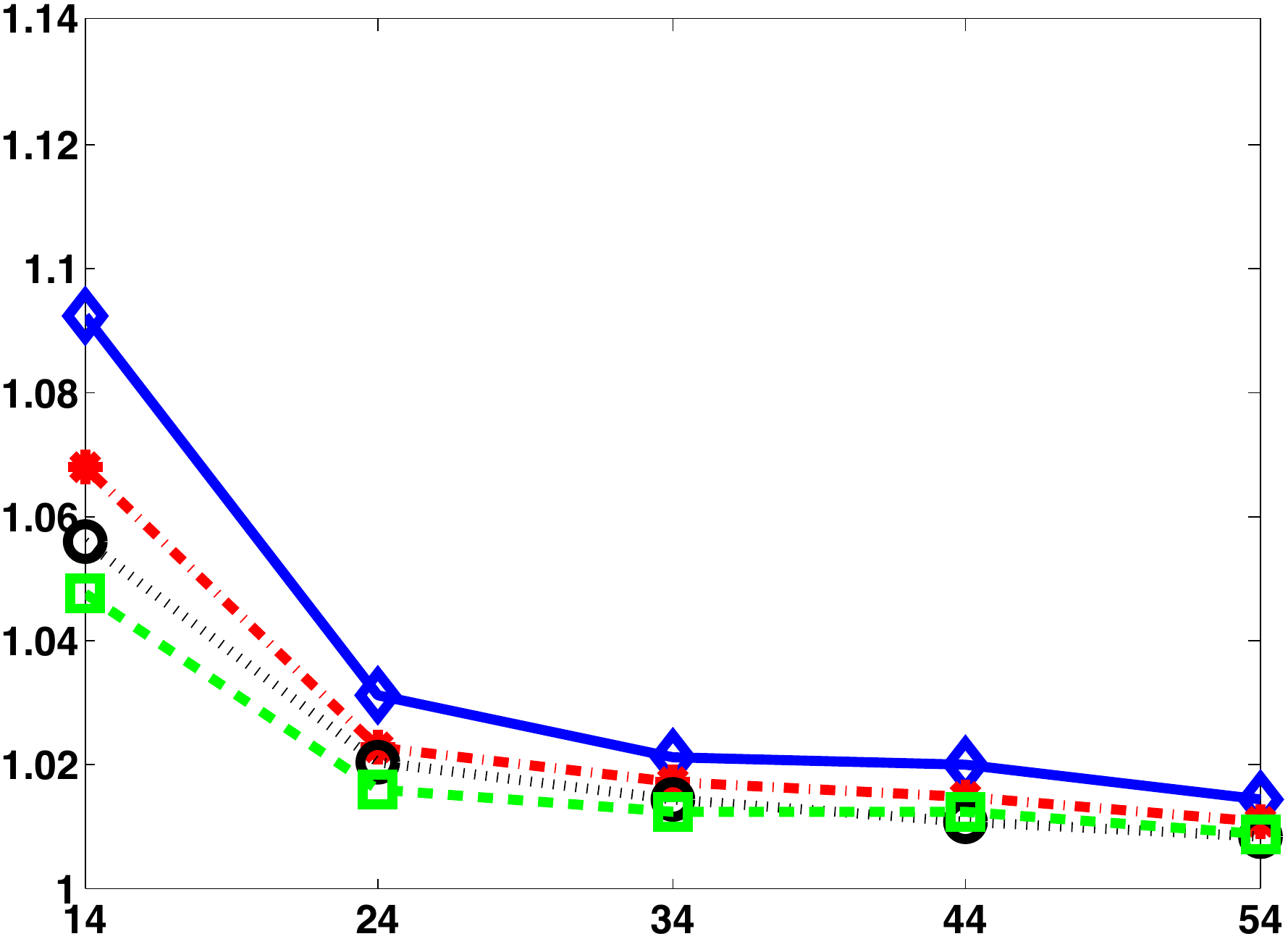}}
    \hspace*{\betweenWidth}
    \subfigure[BOWnytimes]{\includegraphics[width=\sfigWidth,height=\sfigHeight]{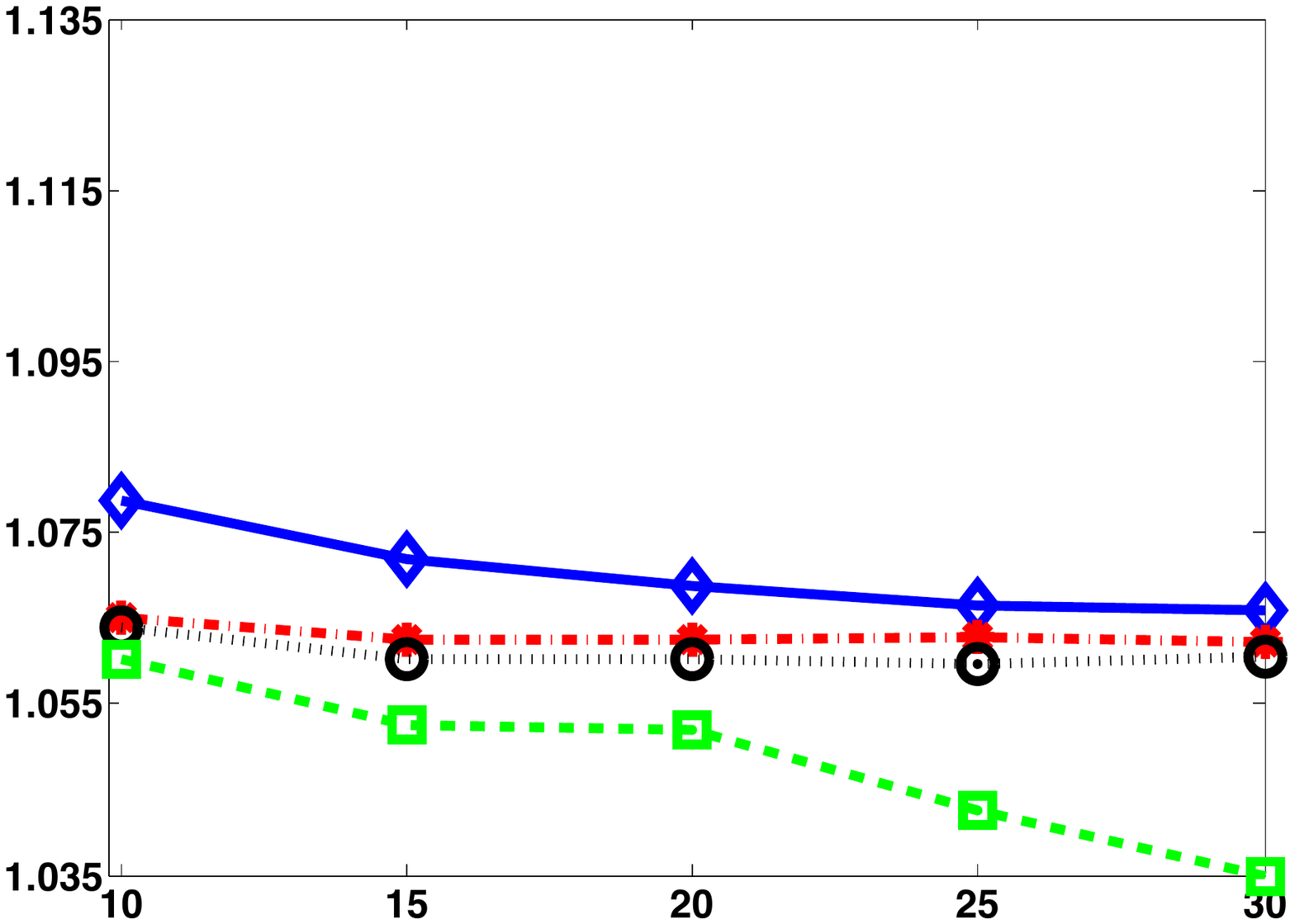}}
    \hspace*{\betweenWidth}
    \subfigure[BOWpubmed]{\includegraphics[width=\sfigWidth,height=\sfigHeight]{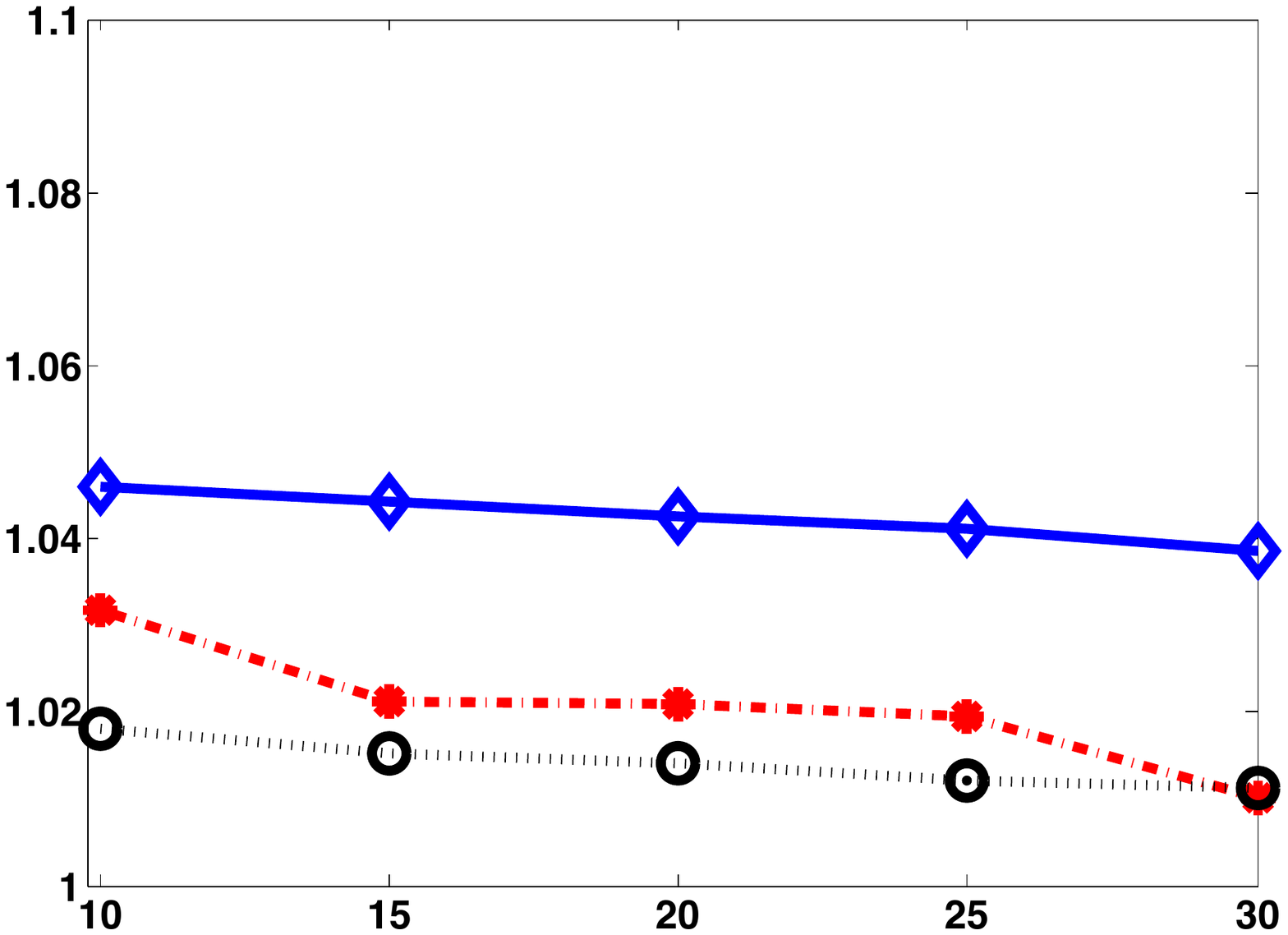}}\\
        \subfigure[NewsGroups]{\includegraphics[width=\sfigWidth,height=\sfigHeight]{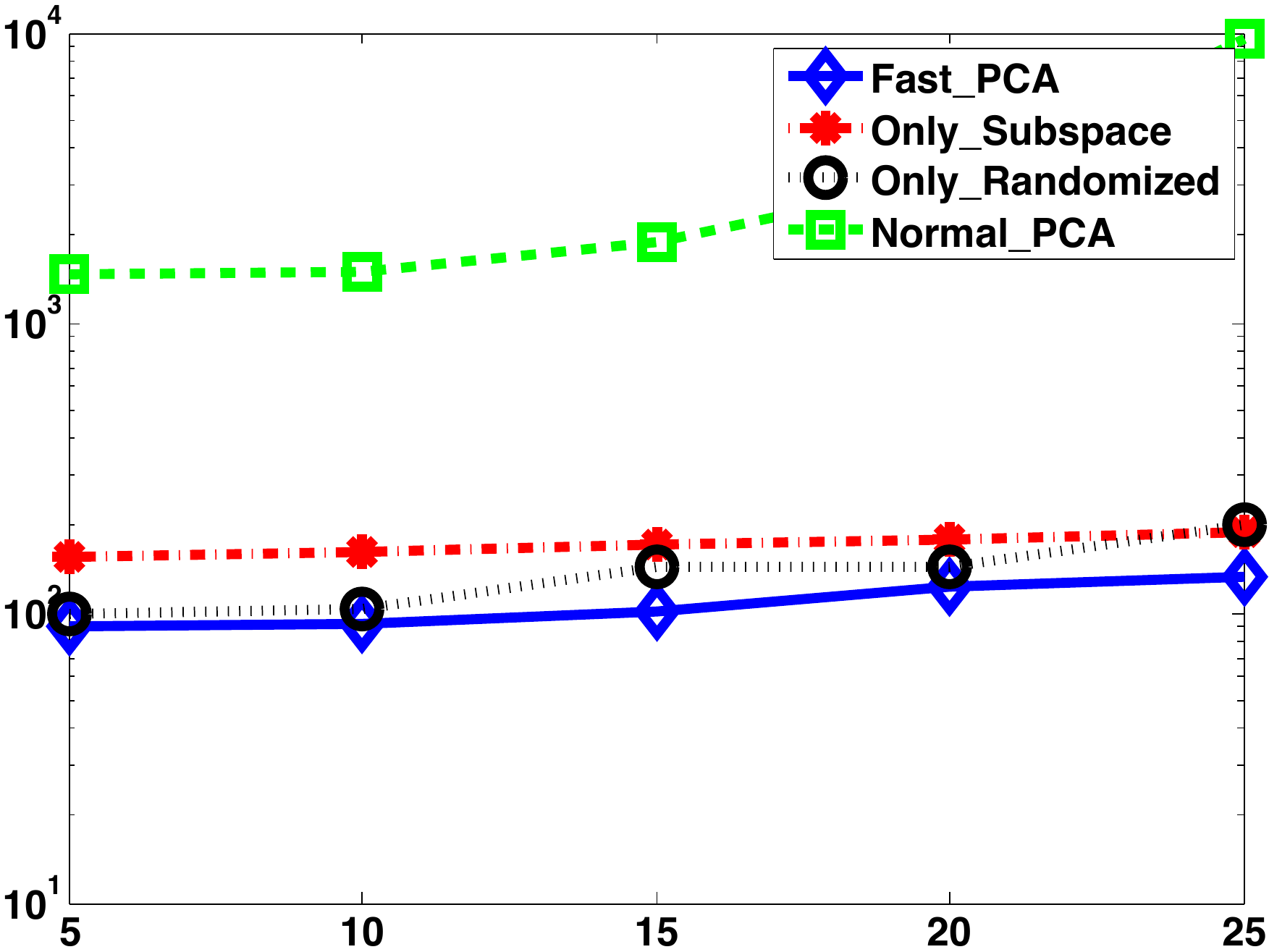}}
    \hspace*{\betweenWidth}
    \subfigure[MNIST]{\includegraphics[width=\sfigWidth,height=\sfigHeight]{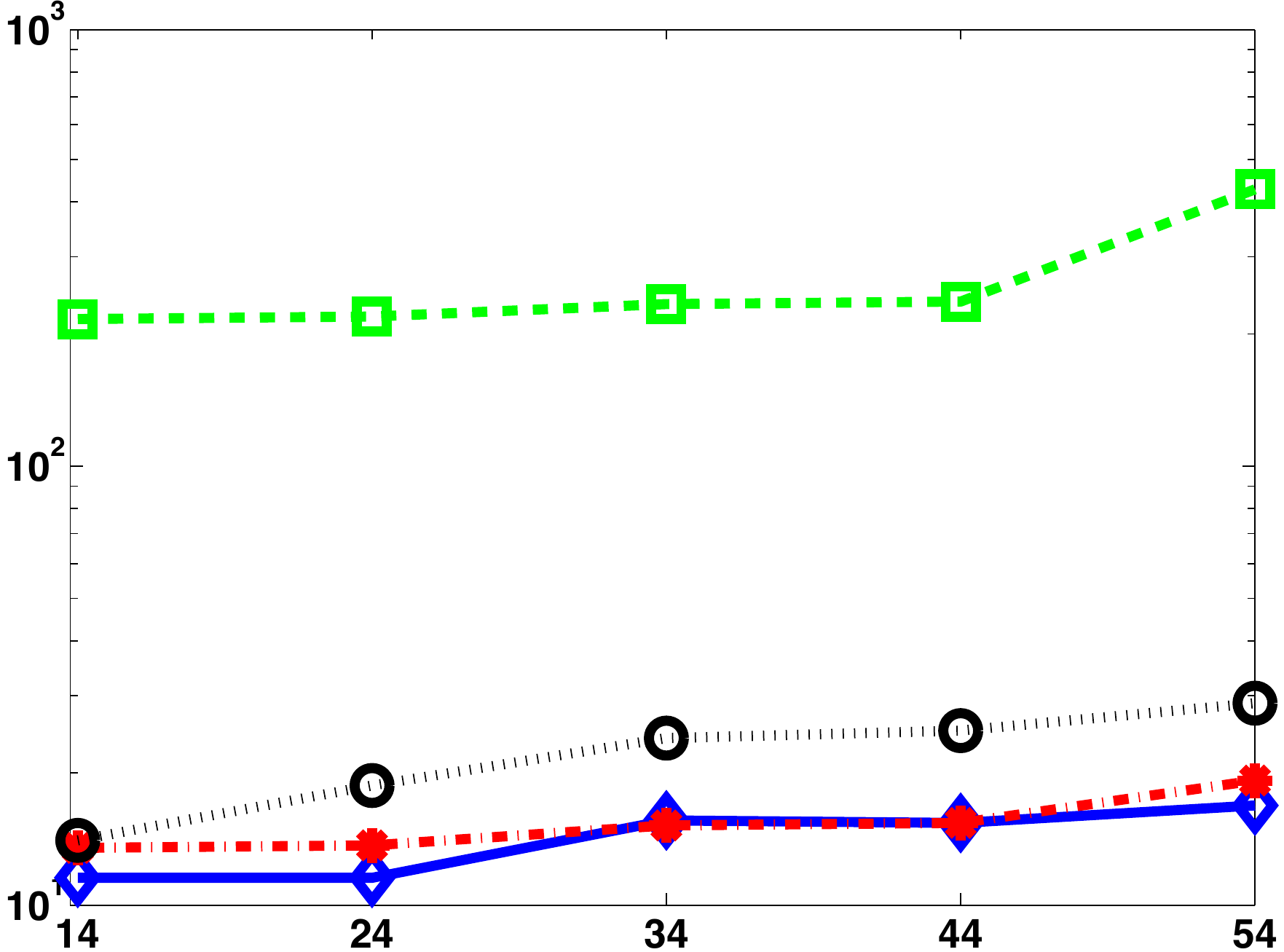}}
    \hspace*{\betweenWidth}
    \subfigure[BOWnytimes]{\includegraphics[width=\sfigWidth,height=\sfigHeight]{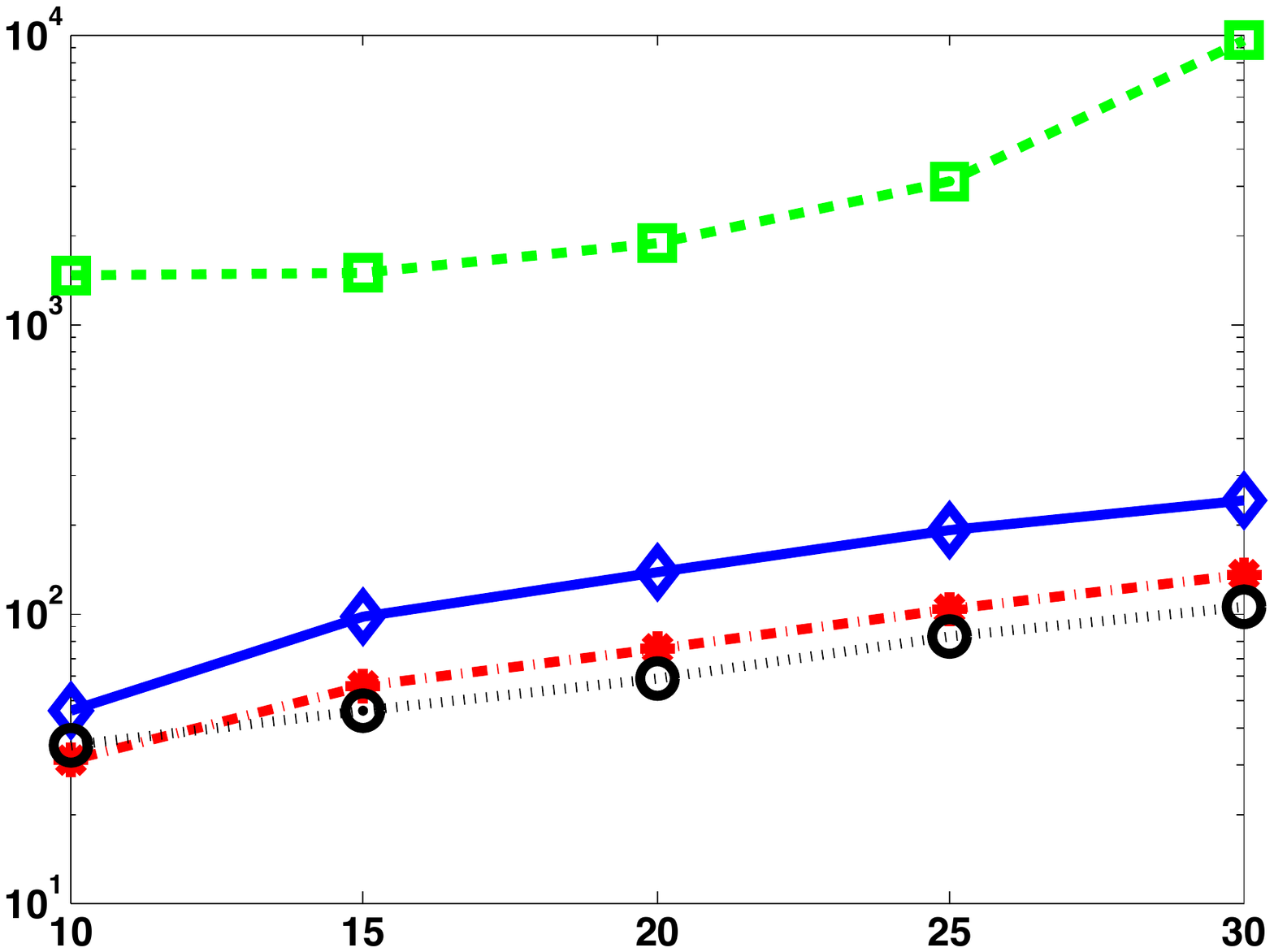}}
    \hspace*{\betweenWidth}
    \subfigure[BOWpubmed]{\includegraphics[width=\sfigWidth,height=\sfigHeight]{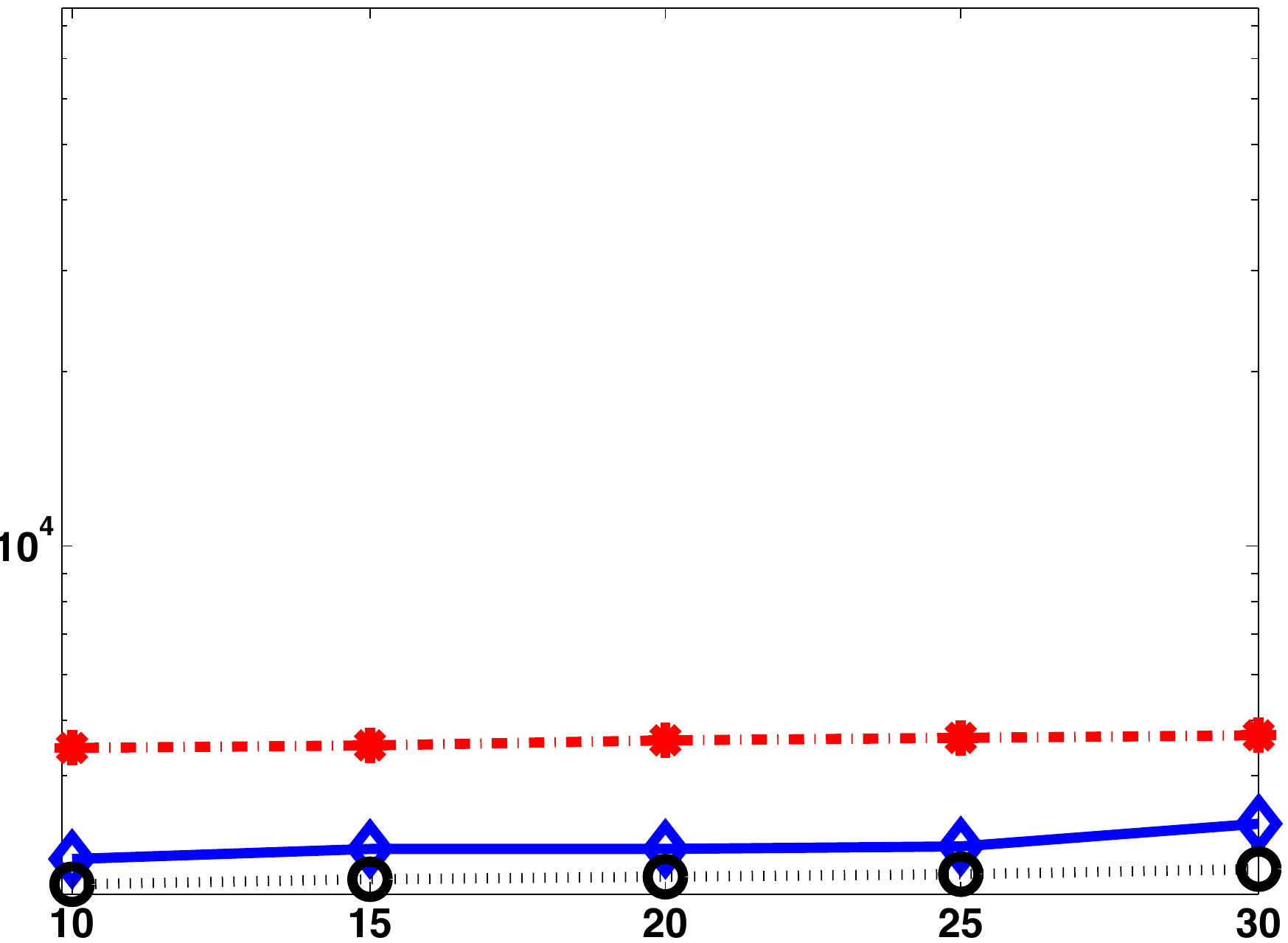}}
\end{center} 
\vspace{-0.3in}
\caption{
{
$k$-means clustering. First row: cost (normalized by baseline) v.s. projection dimension. 
Second row: time v.s. projection dimension. 
}\label{fig:kmeans_cost}
}
\vspace{-0.1in}
\begin{center}
    \subfigure[MNIST]{\includegraphics[width=\sfigWidth,height=\sfigHeight]{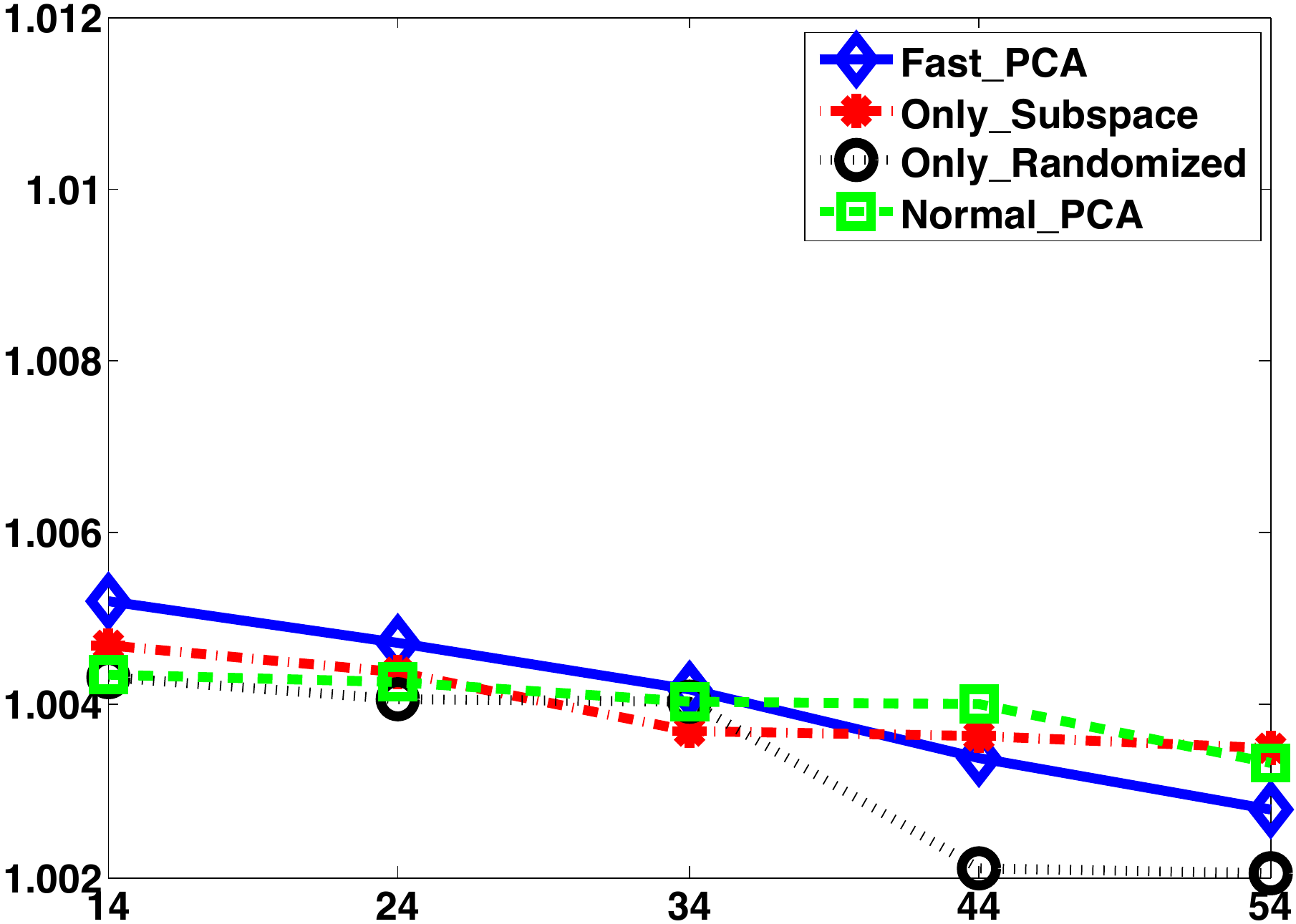}}
    \hspace*{\betweenWidth}
    \subfigure[YearPredictionMSD]{\includegraphics[width=\sfigWidth,height=\sfigHeight]{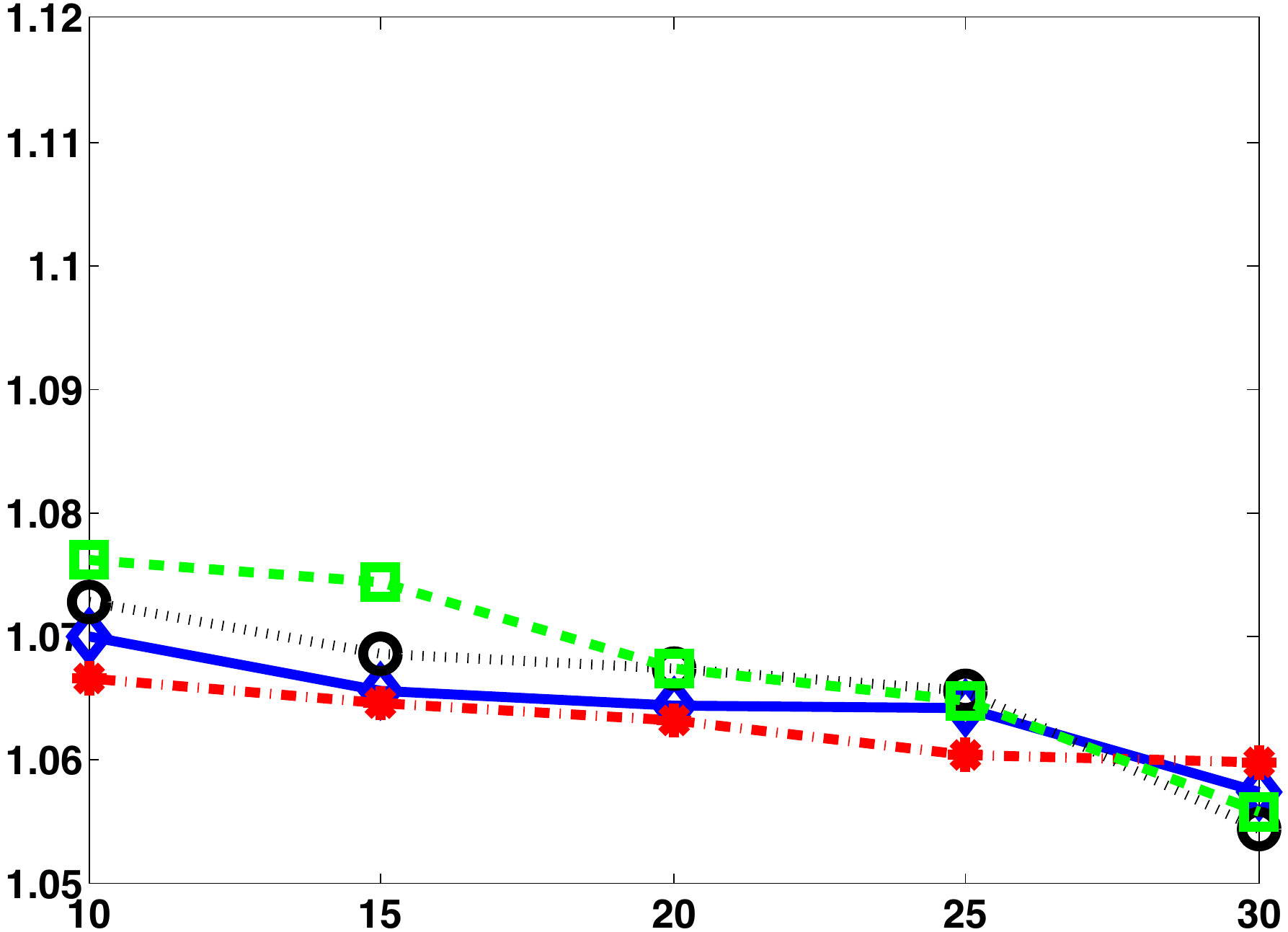}}
    \hspace*{\betweenWidth}
    \subfigure[CTslices]{\includegraphics[width=\sfigWidth,height=\sfigHeight]{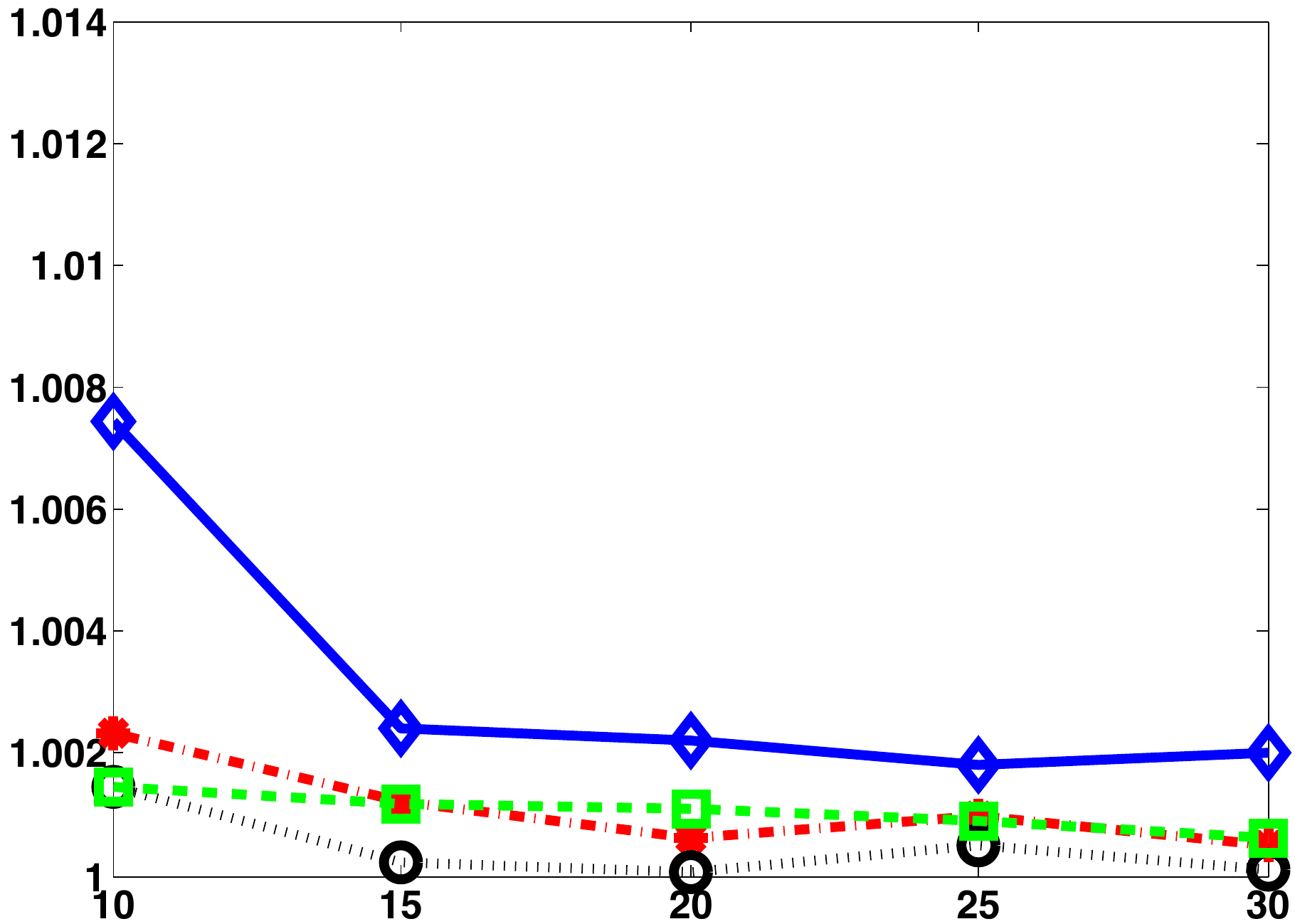}}
    \hspace*{\betweenWidth}
    \subfigure[MNIST8m]{\includegraphics[width=\sfigWidth,height=\sfigHeight]{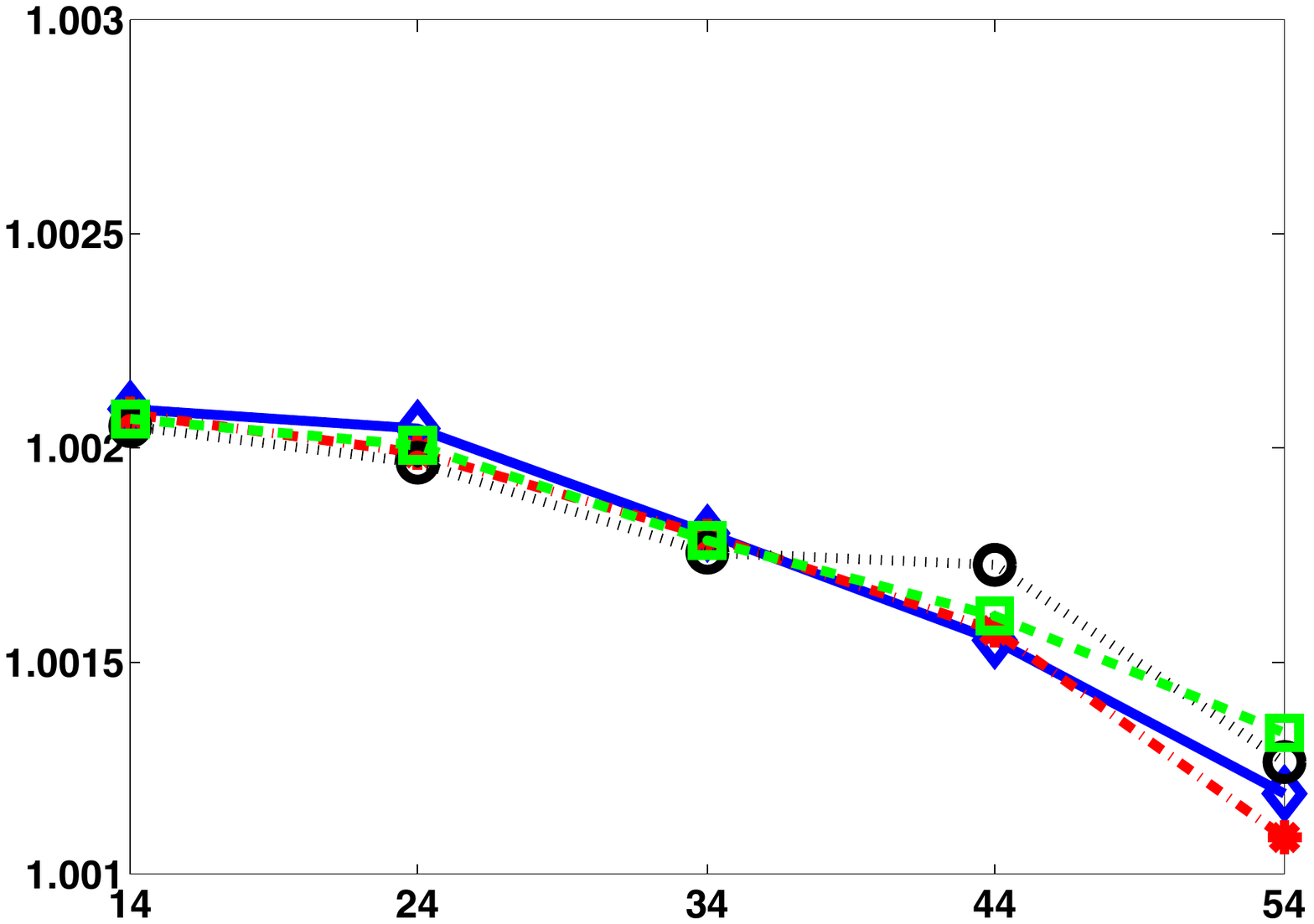}}\\
    \subfigure[MNIST]{\includegraphics[width=\sfigWidth,height=\sfigHeight]{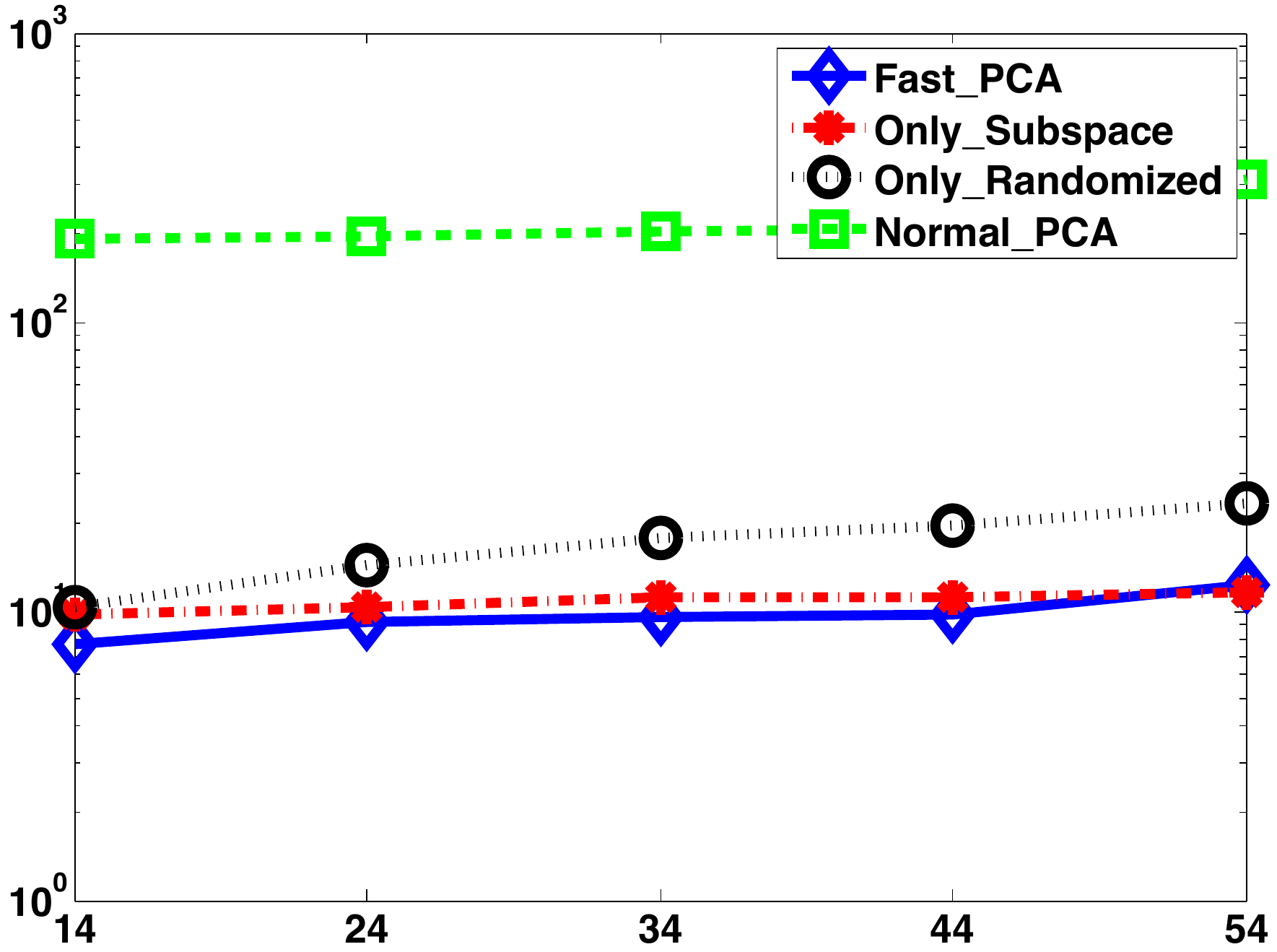}}
    \hspace*{\betweenWidth}
    \subfigure[YearPredictionMSD]{\includegraphics[width=\sfigWidth,height=\sfigHeight]{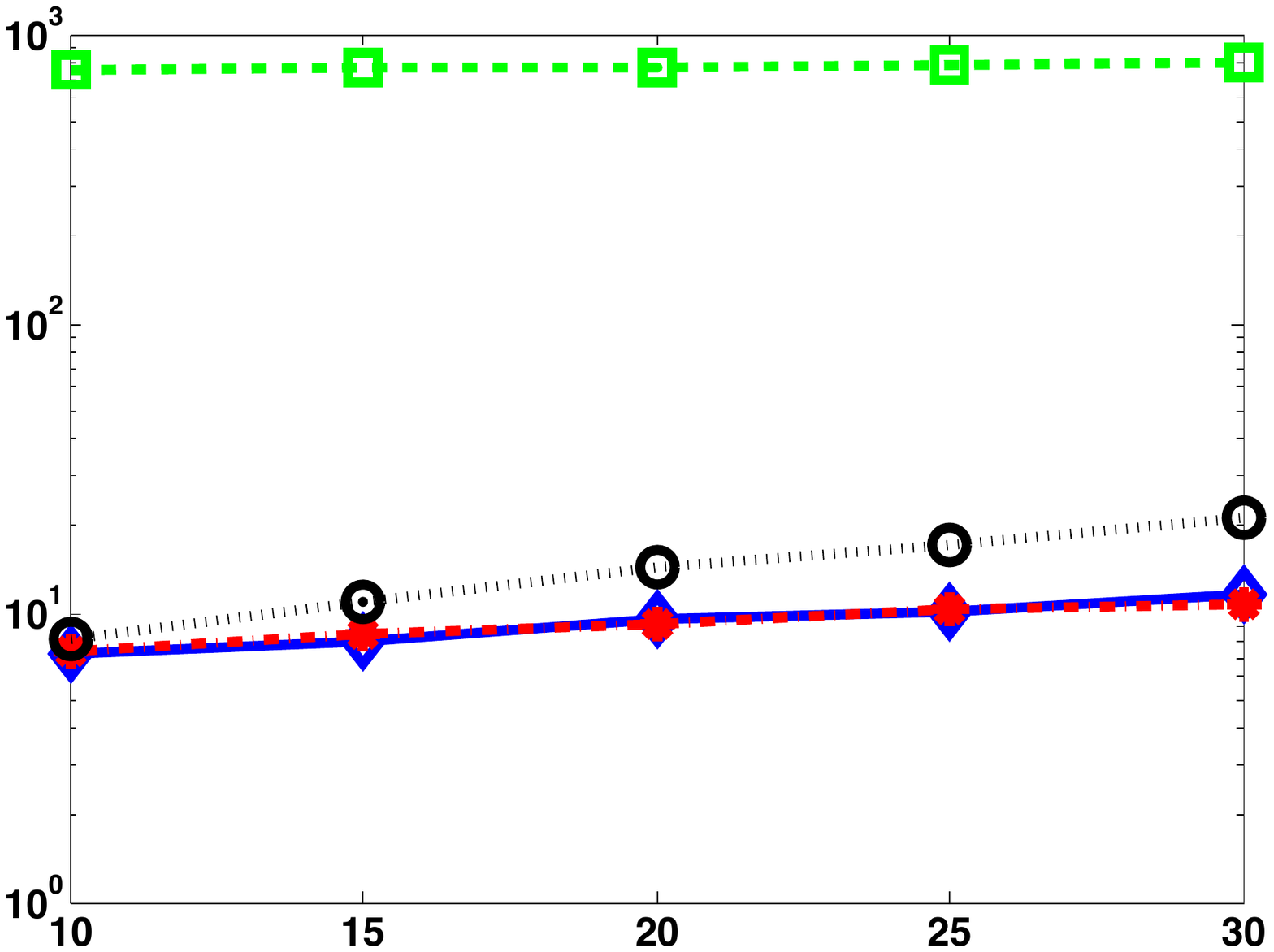}}
    \hspace*{\betweenWidth}
    \subfigure[CTslices]{\includegraphics[width=\sfigWidth,height=\sfigHeight]{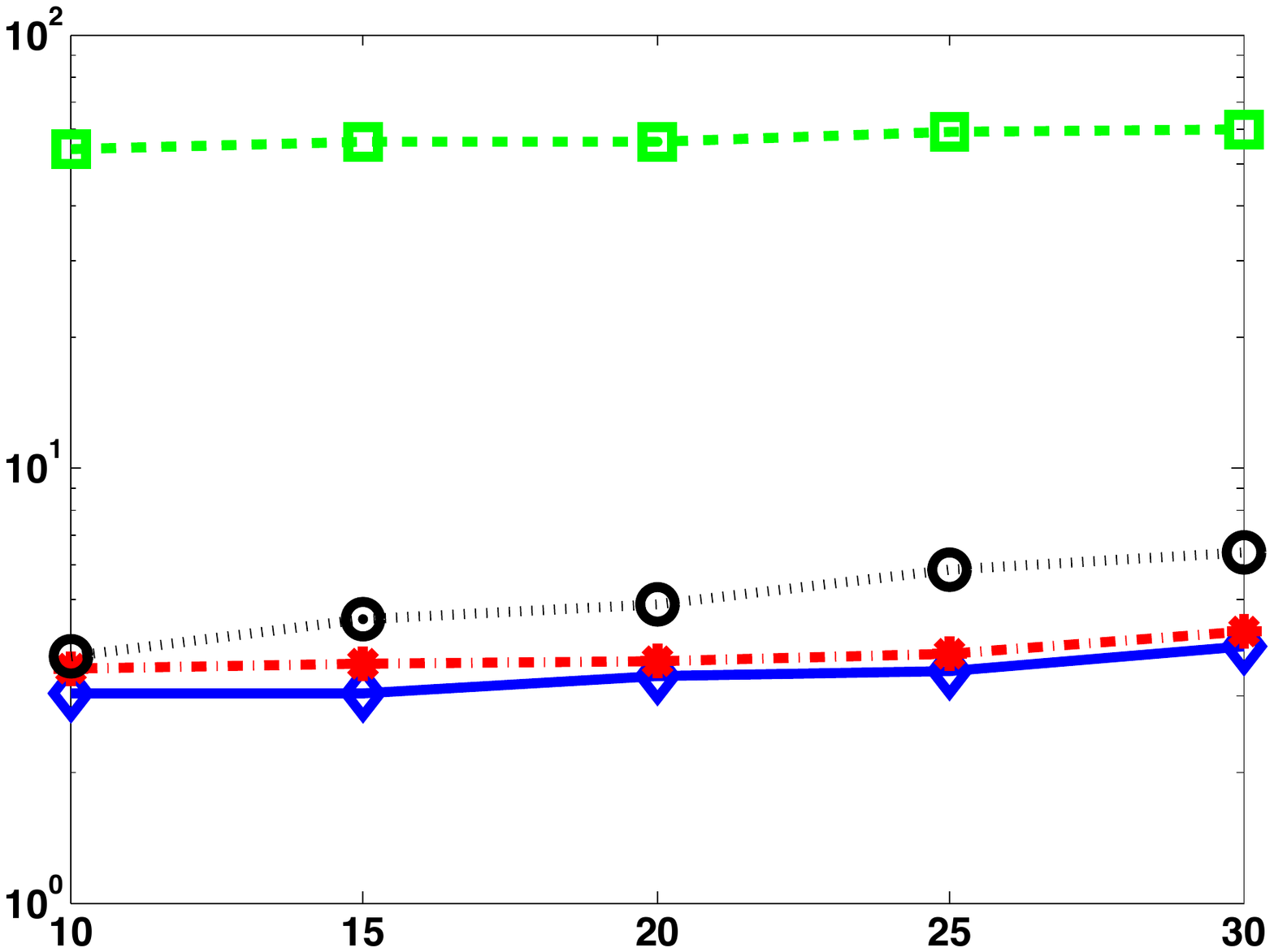}}
    \hspace*{\betweenWidth}
    \subfigure[MNIST8m]{\includegraphics[width=\sfigWidth,height=\sfigHeight]{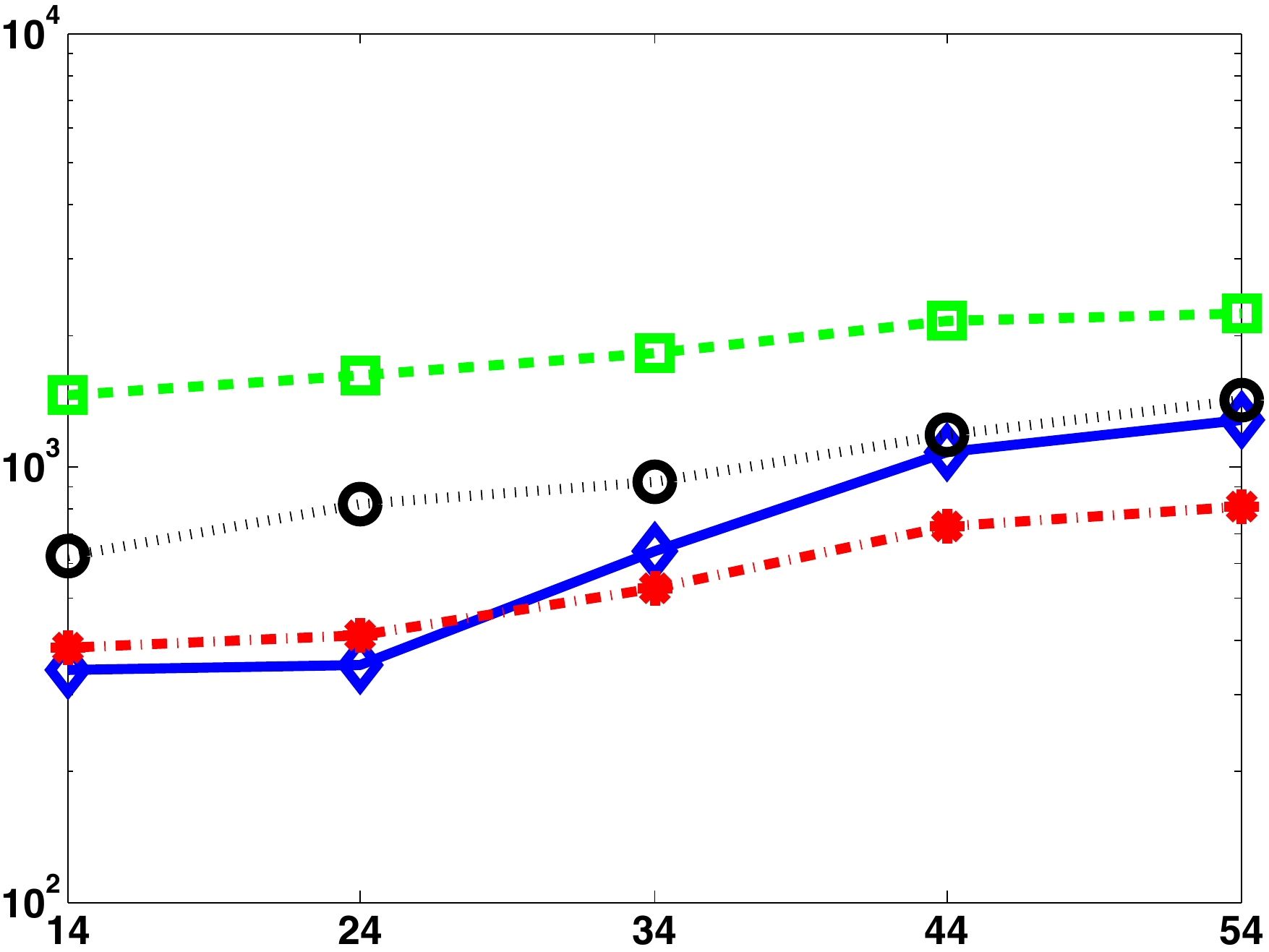}}
\end{center}
\vspace{-0.3in}
\caption{
{
PCR. First row: error (normalized by baseline) v.s. projection dimension. 
Second row: time v.s. projection dimension. 
}\label{fig:pcr_cost}
}
\end{figure*}

\paragraph{Acknowledgments} This work was supported in part by  NSF grants CCF-0953192, CCF-1451177,
CCF-1101283, and CCF-1422910, ONR grant N00014-09-1-0751, and AFOSR grant
FA9550-09-1-0538. David Woodruff would like to acknowledge the XDATA program of the Defense Advanced Research Projects Agency (DARPA), administered through Air Force Research Laboratory contract FA8750-12-C0323, for supporting this work. 

\bibliography{distributedPCA}
\bibliographystyle{plainnat}

\clearpage
\newpage

\appendix

\section{Guarantees for Distributed PCA}

\subsection{Proof of Lemma~\ref{lem:pca_app}}

%

We first prove a generalization of Lemma~\ref{lem:pca_app}.
\begin{lemma}\label{lem:pca_app_gen}
Let $\MA \in \R^{n \times d}$ be an $n \times d$ matrix with singular value decomposition
$\MA = \MU \MD \MV^\top$. Let $\epsilon \in (0,1]$ and $\dimr, t \in \N_+$ with
$d-1 \geq t \geq \dimr + \lceil \dimr/\epsilon \rceil -1$,
and let $\tcd{\MA}=\MA \tc{\MV}{t}(\tc{\MV}{t})^\top$.
Then for any matrix $\MX$ with $d$ rows and $\|\MX\|^2_F\leq r$, we have
\begin{eqnarray*}
\|(\MA - \tcd{\MA} ) \MX\|_F^2 = \|\MA \MX\|_F^2 - \|\tcd{\MA} \MX\|_F^2 \leq  \epsilon \sum_{i=\dimr+1}^{d}\sigma^2_i(\MA).
\end{eqnarray*}
\end{lemma}

\begin{proof}
The proof follows the idea in the proof of Lemma 6.1 in~\citep{dan2013tiny}.

For convenience, let $\overline{\tc{\MD}{t}}$ denote the diagonal matrix that contains the first $t$ diagonal entries in $\MD$ and is 0 otherwise.
Then $\tcd{\MA} = \MU \overline{\tc{\MD}{t}} \MV^\top$
We first have
\begin{eqnarray*}
\|\MA \MX\|_F^2 - \|\tcd{\MA} \MX\|_F^2  &=& \|\MU\MD \MV^\top \MX\|_F^2 - \|\MU \overline{\tc{\MD}{t}} \MV^\top \MX\|_F^2 \\
& = & \|\MD \MV^\top \MX\|_F^2 - \|\overline{\tc{\MD}{t}} \MV^\top \MX\|_F^2\\
& = & \|(\MD-\overline{\tc{\MD}{t}})\MV^\top \MX\|_F^2\\
& = & \|\MU(\MD-\overline{\tc{\MD}{t}})\MV^\top \MX\|_F^2 \\
& = & \|\MA \MX - \tcd{\MA} \MX\|_F^2.
\end{eqnarray*}
where the second and fourth equalities follow since $\MU$ has orthonormal columns, and the third equality follows since for $\Mbf= \MV^\top \MX$
we have
\begin{eqnarray*}
\|\MD\Mbf\|_F^2 - \|\overline{\tc{\MD}{t}}\Mbf\|_F^2 & = & \sum_{i=1}^d \sum_{j=1}^d \sigma^2_i(\MA) m_{ij}^2 - \sum_{i=1}^t \sum_{j=1}^d \sigma^2_i(\MA) m_{ij}^2 \\
& = & \sum_{i=t+1}^d \sum_{j=1}^d \sigma^2_i(\MA) m_{ij}^2 = \|(\MD-\overline{\tc{\MD}{t}})\Mbf\|_F^2 .
\end{eqnarray*}

Next, we bound $\|\MA \MX - \tcd{\MA} \MX\|_F^2$. We have
\begin{eqnarray*}
\|\MA \MX - \tcd{\MA} \MX\|_F^2 = \|(\MD-\overline{\tc{\MD}{t}}) \MV^\top \MX\|_F^2 \leq \|(\MD-\overline{\tc{\MD}{t}})\|_S^2 \|\MX\|_F^2 = \dimr \sigma^2_{t+1}(\MA)
\end{eqnarray*}
where the inequality follows because the spectral norm is consistent with the Euclidean norm.
This implies the lemma since
\begin{eqnarray}
\dimr \sigma^2_{t+1}(\MA) \leq \epsilon (t-\dimr+1)\sigma^2_{t+1}(\MA) \leq \epsilon \sum_{i=\dimr+1}^{t+1}\sigma^2_i(\MA) \leq \epsilon \sum_{i=\dimr+1}^{d}\sigma^2_i(\MA). \label{eqn:data}
\end{eqnarray}
where the first inequality follows for our choice of $t$.
\end{proof}

\noindent
\textbf{Proof of Lemma~\ref{lem:pca_app}}
We are now ready to use Lemma~\ref{lem:pca_app_gen} to prove Lemma~\ref{lem:pca_app}. 
First, any $d\times r$ orthonormal matrix $\MX$ has $\|\MX\|_F^2 \leq r$, so the assumption in  Lemma~\ref{lem:pca_app_gen} is satisfied.
Second,  by the property of the singular value decomposition (Eckart-Young Theorem), we have  
\[
  \sum_{i=\dimr+1}^d \sigma^2_i(\MA) = \min_{\text{rank}(\MB) \leq \dimr} \| \MA - \MB\|_F^2 = \min_{\text{orthonormal } \MY \in \RR^{d\times \dimr }} \| \MA - \MA \MY\MY^\top\|_F^2.
\]
Note that for orthonormal $\MX$, $d^2(\MA, \lspan{\MX}) =  \| \MA - \MA \MX\MX^\top\|_F^2$, so 
$\sum_{i=\dimr+1}^d \sigma^2_i(\MA)  \leq d^2(\MA, \lspan{\MX})$. This completes the proof.

\subsection{Proof of Theorem~\ref{thm:lr_app}}

\noindent
\textbf{Theorem~\ref{thm:lr_app}.}
{\it
Suppose Algorithm {\sf disPCA} 
takes parameters $t_1 \geq \dimr + \lceil 4\dimr/\epsilon \rceil -1$ and $t_2=\dimr$,
and outputs $\tc{\MV}{\dimr}$.
Then
\begin{eqnarray*}
\|\MP - \MP\tc{\MV}{\dimr}(\tc{\MV}{\dimr})^\top\|_F^2 \leq (1+\epsilon) \min_{\MX} d^2(\MP, \lspan{\MX})
\end{eqnarray*}
where the minimization is over $d\times \dimr$ orthonormal matrices $\MX$. The communication is $O(\frac{s\dimr d}{\epsilon})$ words.
}

\begin{proof}
Recall the notations: $\tcd{\MP}_i := \MP_i \MV_i^{(t_1)} (\MV_i^{(t_1)})^\top$ is the data obtained by applying local PCA on local data $P_i$, and $\tcd{\MP}$ is the concatenation of $\tcd{\MP}_i$. Now let $\MX^*$ denote the optimal subspace for $\MP$. 
Our goal is to show that the distance between $\MP$ and the subspace spanned by $\MV^{(r)}$ is close to that between $\MP$ and the subspace spanned by $\MX^*$.

To get some intuition, see Figure~\ref{fig:lowRank} for an illustration. We let $a$ denote the distance between $\MP$ and $L_{\MV^{(r)}}$, that is, $a := d^2(\MP, L_{\MV^{(r)}}) = \|\MP - \MP\tc{\MV}{\dimr}(\tc{\MV}{\dimr})^\top\|_F^2$.
Similarly, let $b$ denote the distance between $\MP$ and $L_{\MX^*}$, $c$ denote that between $\tcd{\MP}$ and $L_{\MV^{(r)}}$, $d$ denote that between $\tcd{\MP}$ and $L_{\MX^*}$.  Then our goal is to show $a-b$ is small. 
Since \[a-b = (a-c ) + (c-d) + (d-b),\] it suffices to bound each of the three terms on the right hand side.

\begin{figure}[!t]
\begin{center}
\input{lowRank}
\end{center}
\caption{
Illustration for the proof of Theorem~\ref{thm:lr_app}.
}\label{fig:lowRank}
\end{figure}

First, we note that the optimal principal components for $\tcd{\MP}$ are $\tc{\MV}{\dimr}$, so $c - d \leq 0$.
This is because $\tcd{\MP} = \tilde{\MU} \MY$ where $\tilde{\MU}$ is a block-diagonal matrix with blocks $\MU_1, \dots, \MU_s$,
and thus the right singular vectors of $\MY$ are also the right singular vectors of $\tcd{\MP}$. 

Now, what is left is to bound $(a-c)$ and $(d-b)$. They are differences between the distances from $\MP$ and $\tcd{\MP}$ to some low dimensional subspace, for which  Lemma~\ref{lem:pca_app} is useful.
Formally, we have the following claim.

\begin{claim}\label{cla:diff}
For any orthonormal matrix $\MX$ of size $d\times \dimr$,
\[d^2(\tcd{\MP}, \lspan{\MX}) - d^2(\MP, \lspan{\MX})  =  \Delta(\MX) - c_0 \]
where $\Delta(\MX) := \|\MP \MX\|_F^2 - \|\tcd{\MP} \MX\|_F^2$ and $c_0 := \|\MP\|_F^2 - \|\tcd{\MP}\|_F^2$.
Furthermore, $$0\leq \Delta(\MX) \leq \epsilon d^2(\MP, \lspan{\MX}),~~c_0 \geq 0.$$
\end{claim}
\begin{proof}
By Pythagorean Theorem,
\begin{eqnarray*}
d^2(\tcd{\MP}, \lspan{\MX}) - d^2(\MP, \lspan{\MX}) =  (\|\tcd{\MP}\|_F^2 - \|\tcd{\MP} \MX\|_F^2) - (\|\MP\|_F^2 - \|\MP \MX\|_F^2 ) = \Delta(\MX) - c_0. 
\end{eqnarray*}
The bound on $\Delta(\MX)$ follows from the fact that 
\[\Delta(\MX) = \|\MP \MX\|_F^2 - \|\tcd{\MP} \MX\|_F^2 = \sum_i [\|\MP_i \MX\|_F^2 - \|\tcd{\MP}_i \MX\|_F^2] \]
and apply Lemma~\ref{lem:pca_app} on each term. The bound on $c_0$ follows from Pythagorean Theorem.
\end{proof}

Applying this claim, we have $a-c = c_0 - \Delta(\MV^{(r)})$ and $d-b =  \Delta(\MV^*) - c_0$, and 
\[(a-c) + (d-b) = \Delta(\MV^*) - \Delta(\MV^{(r)}) \leq \epsilon d^2(\MP, \lspan{\MX^*}).\] This completes the proof.
\end{proof}

\noindent\textbf{Note} A refinement of the proof of Lemma~\ref{lem:pca_app} leads to the following data dependent bound. 

\begin{lemma}\label{lem:data}
The statement in Lemma~\ref{lem:pca_app_gen} holds if $t > \tau(\MA,\dimr,\epsilon)$ where 
$$
	\tau(\MA,\dimr,\epsilon) := \argmin_{t} \left\{\sigma_{t}^2(\MA) \leq \frac{\epsilon}{\dimr} \sum_{i>\dimr} \sigma_i^2(\MA)\right\}.
$$
Furthermore, $\tau(\MA,\dimr,\epsilon) = O(\frac{\dimr}{\epsilon})$.
\end{lemma}
\begin{proof}
Note that the bound on $t$ is only used in proving (\ref{eqn:data}), for which $t > \tau(\MA,\dimr,\epsilon)$ suffices.
$\tau(\MA,\dimr,\epsilon) = O(\frac{\dimr}{\epsilon})$ follows by definition.
\end{proof}

\begin{theorem}\label{thm:lr_app2}
Suppose Algorithm {\sf disPCA} 
takes parameters $t_1 \geq \max_i \tau(\MP_i,\dimr,\epsilon)$ and $t_2=\dimr$,and outputs $\tc{\MV}{\dimr}$.
Then
\begin{eqnarray*}
\|\MP - \MP\tc{\MV}{\dimr}(\tc{\MV}{\dimr})^\top\|_F^2 \leq (1+\epsilon) \min_{\MX} d^2(\MP, \lspan{\MX})
\end{eqnarray*}
where the minimization is over orthonormal matrices $\MX \in \R^{d \times \dimr}$.
The total communication cost is $O(sd \max_i \tau(\MP_i,\dimr,\epsilon))$ words.
\end{theorem}

$\tau(\MP_i,\dimr,\epsilon)$ is typically much less than $O(\dimr/\epsilon)$ in practice. This provides an explanation for the fact that $t_1$ much smaller than $O(\dimr/\epsilon)$ can still lead to good solution for many practical instances. Similar data dependent bounds can be derived for the other theorems in our paper.

\section{Guarantees for Distributed $\ell_2$-Error Fitting}

\subsection{Proof of Lemma~\ref{thm:DisPCA_P_l2}}\label{app:disPCA_closeProj}

Recall that $\tilde{\MP}_i$ denotes the projection of the original data $\MP_i$ to $\tc{\MV}{t}$, and $\tilde{\MP}$ denotes their concatenation.
We further introduce some intermediate variables for our analysis.
Imagine we perform two projections: first project $\MP_i$ to $\tcd{\MP}_i = \MP_i \tc{\MV_i}{t} (\tc{\MV_i}{t})^\top$,
then project $\tcd{\MP}_i$ to $\overline{\MP}_i= \tcd{\MP}_i \tc{\MV}{t} (\tc{\MV}{t})^\top$ where $t=t_1 = t_2$.
Let $\tcd{\MP}$ denote the vertical concatenation of $\tcd{\MP}_i$ and let $\overline{\MP}$ denote the vertical concatenation of $\overline{\MP}_i$, i.e.\
\begin{eqnarray*}
\tcd{\MP} = \left[
                   \begin{array}{c}
                     \tcd{\MP}_1 \\
                     \vdots \\
                     \tcd{\MP}_s \\
                   \end{array}
                 \right]
 \textrm{ \ \ and \ \ }
\overline{\MP} = \left[
                   \begin{array}{c}
                     \overline{\MP}_1 \\
                     \vdots \\
                     \overline{\MP}_s \\
                   \end{array}
                 \right]
\end{eqnarray*}

\noindent
\textbf{Lemma~\ref{thm:DisPCA_P_l2}.}
{\it
Let $t_1=t_2 \geq k + \lceil 8k/\epsilon \rceil -1$ in Algorithm {\sf disPCA}  
for $k\in \N_+$ and $\epsilon \in (0,1)$.
Then for any $d \times k$ matrix $\MX$ with orthonormal columns,
\begin{eqnarray}
0 \leq & \|\MP \MX - \gpca{\MP} \MX\|_F^2  & \leq \epsilon d^2(\MP, \lspan{\MX}),\label{eqn:dispca_eq1}\\
0 \leq & \|\MP \MX\|_F^2 - \|\gpca{\MP} \MX\|_F^2 & \leq \epsilon d^2(\MP, \lspan{\MX}) \label{eqn:dispca_eq2}.
\end{eqnarray}
}

\begin{proof}
Before going to the proof, we note that unlike in Lemma~\ref{lem:pca_app_gen},  $\|\MP \MX - \gpca{\MP} \MX\|_F^2$ may not equal $ \|\MP \MX\|_F^2 - \|\gpca{\MP} \MX\|_F^2$ since multiple SVD are applied.

For the first statement (\ref{eqn:dispca_eq1}), we have
\begin{eqnarray}
\|\MP \MX - \gpca{\MP} \MX\|_F^2 & \leq & 2 \|\MP \MX-\tcd{\MP} \MX\|_F^2 \label{eqn:dispca_1term1}\\
&+ & 2\|\tcd{\MP} \MX - \overline{\MP} \MX\|_F^2 \label{eqn:dispca_1term2}\\
& +&  2\|\overline{\MP} \MX - \gpca{\MP} \MX\|_F^2. \label{eqn:dispca_1term3}
\end{eqnarray}
For (\ref{eqn:dispca_1term1}), we have by Lemma~\ref{lem:pca_app_gen}
\begin{eqnarray}
\|\MP \MX-\tcd{\MP} \MX\|_F^2 = \sum_{i=1}^s \|\MP_i \MX-\tcd{\MP}_i \MX\|_F^2  \leq \sum_{i=1}^s \frac{\epsilon}{4} d^2(\MP_i, \lspan{\MX}) = \frac{\epsilon}{8} d^2(\MP, \lspan{\MX}). \label{eqn:dispca_1bound1}
\end{eqnarray}
Similarly, for (\ref{eqn:dispca_1term2}) we have by Lemma~\ref{lem:pca_app_gen}
\begin{eqnarray}
\|\tcd{\MP} \MX-\overline{\MP} \MX\|_F^2 \leq \frac{\epsilon}{8} d^2(\tcd{\MP}, \lspan{\MX}). \label{eqn:dispca_1bound2}
\end{eqnarray}
To bound (\ref{eqn:dispca_1term3}), let $\MY=\tc{\MV}{t} (\tc{\MV}{t})^\top \MX$.
Then by definition, $\overline{\MP}_i \MX = \tcd{\MP}_i \MY$ and $\gpca{\MP}_i\MX = \MP_i \MY$.
By Lemma~\ref{lem:pca_app_gen}, we have
\begin{eqnarray}
\|\overline{\MP} \MX-\gpca{\MP} \MX\|_F^2 & = & \sum_{i=1}^s \|\tcd{\MP}_i \MY-\MP_i \MY\|_F^2 \\
& \leq & \sum_{i=1}^s \frac{\epsilon}{8} \sum_{i=\dimr+1}^s \sigma^2_i(\MP_i) \leq \frac{\epsilon}{8} \sum_{i=1}^s d^2(\MP_i,\lspan{\MX}) = \frac{\epsilon}{8} d^2(\MP,\lspan{\MX}). \label{eqn:dispca_1bound3}
\end{eqnarray}
Combining (\ref{eqn:dispca_1bound1})(\ref{eqn:dispca_1bound2}) and (\ref{eqn:dispca_1bound3}) leads to
\begin{eqnarray}
\|\MP \MX-\gpca{\MP} \MX\|_F^2 \leq  \frac{\epsilon}{2} d^2(\MP, \lspan{\MX}) + \frac{\epsilon}{4} d^2(\tcd{\MP}, \lspan{\MX}). \label{eqn:dispca_1bound123}
\end{eqnarray}

We now only need to bound $d^2(\tcd{\MP}, \lspan{\MX})$ is similar to $d^2(\MP, \lspan{\MX})$, which is done in Claim~\ref{cla:diff}. The first statement then follows.

For the second statement (\ref{eqn:dispca_eq2}), we have a similar argument.
\begin{eqnarray}
\|\MP \MX\|_F^2 - \|\gpca{\MP} \MX\|_F^2 & =& \|\MP \MX\|_F^2 - \|\tcd{\MP} \MX\|_F^2 \label{eqn:dispca_term1}\\
&+ & \|\tcd{\MP} \MX\|_F^2 - \|\overline{\MP} \MX\|_F^2 \label{eqn:dispca_term2}\\
& +&  \|\overline{\MP} \MX\|_F^2 - \|\gpca{\MP} \MX\|_F^2. \label{eqn:dispca_term3}
\end{eqnarray}
For (\ref{eqn:dispca_term1}), we have by Lemma~\ref{lem:pca_app_gen}
\begin{eqnarray}
\|\MP \MX\|_F^2 - \|\tcd{\MP} \MX\|_F^2 = \sum_{i=1}^s \left[\|\MP_i \MX\|_F^2 - \|\tcd{\MP}_i \MX\|_F^2 \right] \leq \sum_{i=1}^s \frac{\epsilon}{4} d^2(\MP_i, \lspan{\MX}) = \frac{\epsilon}{4} d^2(\MP, \lspan{\MX}). \label{eqn:dispca_bound1}
\end{eqnarray}
Similarly, for (\ref{eqn:dispca_term2}) we have by Lemma~\ref{lem:pca_app_gen}
\begin{eqnarray}
\|\tcd{\MP} \MX\|_F^2 - \|\overline{\MP} \MX\|_F^2 \leq \frac{\epsilon}{4} d^2(\tcd{\MP}, \lspan{\MX}). \label{eqn:dispca_bound2}
\end{eqnarray}

By Lemma~\ref{lem:pca_app_gen}, we have
\begin{eqnarray}
\|\overline{\MP} \MX\|_F^2 - \|\gpca{\MP} \MX\|_F^2 & = & \sum_{i=1}^s \left[\|\tcd{\MP}_i \MY\|_F^2 - \|\MP_i \MY\|_F^2\right] \nonumber\\
& \leq & \sum_{i=1}^s \frac{\epsilon}{4} \sum_{i=\dimr+1}^s \sigma^2_i(\MP_i) \leq \frac{\epsilon}{4} \sum_{i=1}^s d^2(\MP_i,\lspan{\MX}) = \frac{\epsilon}{4} d^2(\MP,\lspan{\MX}). \label{eqn:dispca_bound3}
\end{eqnarray}

Combining (\ref{eqn:dispca_bound1})(\ref{eqn:dispca_bound2}) and (\ref{eqn:dispca_bound3}) leads to
\begin{eqnarray}
\|\MP \MX\|_F^2 - \|\gpca{\MP} \MX\|_F^2 \leq  \frac{\epsilon}{2} d^2(\MP, \lspan{\MX}) + \frac{\epsilon}{4} d^2(\tcd{\MP}, \lspan{\MX}). \label{eqn:dispca_bound123}
\end{eqnarray}
The second statement then follows from (\ref{eqn:dispca_bound123}) and Claim~\ref{cla:diff}.
\end{proof}

%
%
%

\subsection{Proof of Theorem~\ref{thm:coreset_gen}}\label{app:disPCA_coreset}

The following weak triangle inequality is useful for our analysis.
\begin{fact}\label{fac:weakTri}
For any $a,b\in \R$ and $\epsilon \in (0,1)$,
$|a^2 - b^2| \leq \frac{3(a-b)^2}{\epsilon} + 2\epsilon a^2$.
\end{fact}
\begin{proof}
Either $|a| \leq \frac{|a-b|}{\epsilon}$ or $|a-b| \leq \epsilon |a|$, so we have
$|a| |a-b| \leq \frac{(a-b)^2}{\epsilon} + \epsilon a^2$.
This leads to
\begin{eqnarray*}
|a^2-b^2|  =  |a-b| |a+b| \leq |a-b| (|2a| + |b-a|) = 2 |a| |a-b|  + (a-b)^2 \leq \frac{2(a-b)^2}{\epsilon} + 2\epsilon a^2 + (a-b)^2
\end{eqnarray*}
which completes the proof.
\end{proof}

We first prove the theorem for the special case of $k$-means clustering, and the same argument leads to the guarantee for general $l_2$-error fitting problems. Note  that because we use the weak triangle inequality, we lose a factor of $1/\epsilon$. Thus, we require $t_1=t_2 = O(k/\epsilon^2)$, instead of $O(k/\epsilon)$ as in Lemma~\ref{thm:DisPCA_P_l2}.

\begin{theorem}~\label{thm:coreset}
Let $t_1=t_2 \geq k + \lceil 4k/\epsilon^2 \rceil -1$ in Algorithm {\sf disPCA}.
Then there exists a constant $c_0 \geq 0$, such that
for any set of $k$ points $\x$,
$$(1-\epsilon) d^2(\MP,\x) \leq d^2(\gpca{P}, \x) + c_0 \leq (1+\epsilon)
d^2(\MP,\x).$$
\end{theorem}

\begin{proof}
The proof follows that in~\citep{dan2013tiny}, with slight modification for the distributed setting.

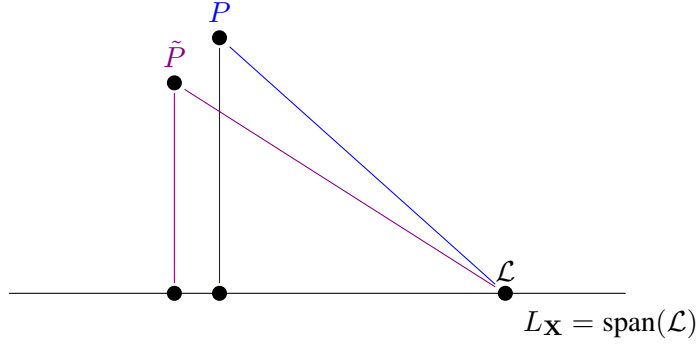
\begin{figure}[!t]
\begin{center}
\input{disPCAMotivate2}
\end{center}
\caption{
Illustration for the proof of Theorem~\ref{thm:coreset}.
}\label{fig:disPCAMotivate2}
\end{figure}

Let $\MX \in \R^{d \times k}$ has orthonormal columns that span $\x$;
see Figure~\ref{fig:disPCAMotivate2} for an illustration.
Then the costs of $\MP$ and $\gpca{P}$ can be decomposed into two parts: one part is from $\MP$ (or $\gpca{P}$)
to its projection on $L_\MX$, and the other part is from the projection to the centers.
Then we can compare the two parts separately.

Let $\gpcap{p}_i$ be the point in $\gpca{P}$ corresponding to $p_i$ in $\MP$.
Let $c_0 = \|\MP\|_F^2 - \|\gpca{P}\|_F^2$.
Then by Pythagorean theorem we have
\begin{eqnarray*}
|d^2(\MP,\x) - d^2(\gpca{P}, \x) -  c_0| \leq \biggl |d^2(\MP,L_{\MX}) - d^2(\gpca{P}, \lspan{\MX})  - c_0 \biggr|  +\biggl|\sum_{i =1}^{|\MP|} \bigl[d(\proj_\MX(p_i), \x)^2 - d(\proj_\MX(\gpcap{p}_i), \x)^2 \bigr]\biggr|.
\end{eqnarray*}

For the first part, we have by Pythagorean theorem
\begin{eqnarray}
d^2(\MP,L_{\MX}) - d^2(\gpca{P}, \lspan{\MX}) - c_0 = (\|\MP\|_F^2 -\|\MP \MX\|_F^2)- (\|\gpca{P}\|_F^2 - \|\gpca{P}\MX\|_F^2) - c_0 =  \|\gpca{P}\MX\|_F^2  - \|\MP \MX\|_F^2.\label{eqn:coreset_bound1}
\end{eqnarray}

For the second part, by Fact~\ref{fac:weakTri} we have
\begin{eqnarray}
\sum_{i =1}^{|\MP|} \left|d(\proj_\MX(p_i), \x)^2 - d(\proj_\MX(\gpcap{p}_i), \x)^2\right|  & \leq & \sum_{i=1}^{|\MP|} \left[\frac{12 d(\proj_\MX(p_i), \proj_\MX(\gpcap{p}_i))^2}{\epsilon} + \frac{\epsilon}{2}d(\proj_\MX(p_i), \x)^2\right]\nonumber\\
& = & \frac{12}{\epsilon}\|(\MP-\gpca{P})\MX\|_F^2 + \frac{\epsilon}{2} \sum_{i=1}^{|\MP|} d(\proj_\MX(p_i), \x)^2 \nonumber\\
& \leq & \frac{12}{\epsilon}\|(\MP-\gpca{P})\MX\|_F^2 + \frac{\epsilon}{2} \sum_{i=1}^{|\MP|} d(p_i, \x)^2.\label{eqn:coreset_bound2}
\end{eqnarray}
We first note that $d^2(\MP, \lspan{\MX}) \leq d^2(\MP,\x)$. 
For the other terms in (\ref{eqn:coreset_bound1})(\ref{eqn:coreset_bound2}),
we need to use  Lemma~\ref{thm:DisPCA_P_l2} with accuracy $\epsilon^2$ (instead of $\epsilon$).
This then leads to the theorem.
\end{proof}

The general statement for $\ell_2$-error geometric fitting problems follows from the same argument.

\noindent
\textbf{Theorem~\ref{thm:coreset_gen}.}
{\it
Let $t_1=t_2 = O(\dimr k/ \epsilon^2)$ in Algorithm {\sf disPCA} 
for $\epsilon \in (0, 1/3)$. Then there exists a constant $c_0 \geq 0$ such that for any set of $k$ centers $\x$ in $\dimr$-Subspace $k$-Clustering,
$$
	(1-\epsilon) d^2(\MP,\x) \leq d^2(\gpca{\MP}, \x) + c_0 \leq (1+\epsilon) d^2(\MP,\x).
$$
}

\section{Fast Distributed PCA}

\subsection{Proofs for Subspace Embedding}

\begin{algorithm}[t]
\caption{Fast Sparse Subspace Embedding~\citep{clarkson2013low}}
\label{alg:sparseEmbedding}
\begin{algorithmic}[1]
\REQUIRE{parameters $n,\ell\in \N_+$.}
\STATE{Let $h: [n] \mapsto [\ell]$ be a random map, so that for each $i \in [n], h(i)=j$ for $j\in [\ell]$ with probability $1/\ell$.}
\STATE{Let $\MPhi$ be an $\ell\times n$ binary matrix with $\MPhi_{h(i),i}=1$, and all remaining entries $0$.}
\STATE{Let $\MD$ be an $n\times n$ diagonal matrix, with each diagonal entry independently chosen as $+1$ or $-1$ with equal probability.}
\ENSURE{$\MH=\MPhi \MD$.}
\end{algorithmic}
\end{algorithm}

The construction of the embedding matrix $\MH$ is presented in Algorithm~\ref{alg:sparseEmbedding}.
Note that the embedding matrix $\MH$ does not need to be built explicitly;
we can compute the embedding $\MH\MA$ for an given matrix $\MA$ in a direct and faster way.
Algorithm~\ref{alg:sparseEmbedding} has the following guarantee. 

\begin{theorem}\citep{clarkson2013low,meng2013low,nelson2012osnap}\label{thm:sparseEmbedding}
Suppose $n>d$ and $\ell = O(\frac{d^2}{\epsilon^2})$.
With probability at least $99/100$,
$\|\MH \MA y\|_2 = (1 \pm \epsilon) \|\MA y\|_2$ for all vectors $y \in \R^d$.
Moreover, $\MH\MA$ can be computed in time $O(\nnz(\MA))$ where $\nnz(\MA)$ is the number of non-zero entries in $\MA$.
\end{theorem}

\begin{lemma}\label{lem:dis_pca_subspace}
Let $\epsilon \in (0,1/2]$ and $k, t \in \N_+$ with
$d-1 \geq t \geq k + \lceil 4k/\epsilon \rceil -1$.
Suppose Algorithm {\sf disPCA} 
takes input $\{\MH_i\MP_i\}_{i=1}^s$
and outputs $\tc{\MV}{t}$.
Let $\gpca{P}=\MP\tc{\MV}{t}(\tc{\MV}{t})^\top$.
Then for any $d \times k$ matrix $\MX$ with orthonormal columns,
\begin{eqnarray*}
\|\MP \MX - \gpca{\MP} \MX\|_F^2 & \leq & \epsilon d^2(\MP, \lspan{\MX}),\\
\bigl| \|\MP \MX\|_F^2 - \|\gpca{\MP} \MX\|_F^2 \bigr|& \leq  & 3\epsilon\|\MP\MX\|_F^2 + \epsilon d^2(\MP,
\lspan{\MX}).
\end{eqnarray*}
\end{lemma}

\begin{proof}
First note that
the input to Algorithm {\sf disPCA} 
is $\MT \MP$ where $\MT$ is
a block-diagonal matrix with blocks $\MH_1, \dots, \MH_s$.
Then the projection of the input to $\tc{\MV}{t}$ is
$\MT\MP \tc{\MV}{t}(\tc{\MV}{t})^\top=\MT \gpca{P}$.
By Lemma~\ref{thm:DisPCA_P_l2}, for any $d \times k$ matrix $\MX$ with orthonormal
columns, we have
\begin{eqnarray}
0\leq & \|\MT \MP \MX - \MT \gpca{\MP} \MX\|_F^2  & \leq \frac{\epsilon}{4} d^2(\MT \MP,\lspan{\MX}),\label{eqn:boundtp1}\\
0 \leq & \|\MT \MP \MX\|_F^2 - \|\MT \gpca{\MP} \MX\|_F^2 & \leq \frac{\epsilon}{4} d^2(\MT \MP,\lspan{\MX}).\label{eqn:boundtp2}
\end{eqnarray}
By properties of $\MT$, we have
\begin{eqnarray*}
\|\MT \MP \MX - \MT \gpca{\MP} \MX\|_F^2 = \|\MT (\MP \MX - \gpca{\MP} \MX)\|_F^2 \geq (1 - \epsilon) \|\MP \MX - \gpca{\MP} \MX\|_F^2
\end{eqnarray*}
and
\begin{eqnarray*}
d^2(\MT \MP,\lspan{\MX}) = \|\MT \MP-\MT \MP\MX\MX^\top\|_F^2\leq (1+\epsilon) \|\MP- \MP\MX\MX^\top\|_F^2 = (1+\epsilon) d^2(\MP,\lspan{\MX}).
\end{eqnarray*}
Combined with (\ref{eqn:boundtp1}), these lead to the first claim.

Similarly, we also have $\|\MT\MP\MX\|_F^2 = (1\pm \epsilon) \|\MP\MX\|_F^2$
and $\|\MT\gpca{P}\MX\|_F^2 = (1\pm \epsilon)
\|\gpca{P}\MX\|_F^2$.
Plugging these into (\ref{eqn:boundtp2}), we obtain
$$-3\epsilon\|\MP\MX\|_F^2 \leq \|\MP \MX\|_F^2 - \|\gpca{\MP} \MX\|_F^2 \leq 3\epsilon\|\MP\MX\|_F^2 + \epsilon d^2(\MP,
\lspan{\MX})$$
which establishes the lemma.
\end{proof}

\begin{algorithm}[t]
\caption{Boosting success probability of embedding}
\label{alg:success}
\begin{algorithmic}[1]
\REQUIRE{$\MA \in \R^{n\times d}$, parameters $\epsilon, \delta$.}
\STATE{Construct $r=O(\log \frac{1}{\delta})$ independent subspace embeddings $\MH_j \MA$,
each having accuracy $\epsilon/9$ and success probability $99/100$.}
\STATE{Compute SVD $\MH_j\MA = \MU_j \MD_j \MV_j^\top$ for $j\in [r]$.}
\FOR{$j \in [r]$}
    \STATE{Check if for at least half $j' \neq j$, $$\sigma_i(\MD_{j'} \MV_{j'}^\top \MV_j \MD_j^{-1}) \in [1\pm \epsilon/3],\forall i.$$}
    \STATE{If so, output $\MH_j \MA$.}
\ENDFOR
\end{algorithmic}
\end{algorithm}

\begin{theorem}\label{thm:success}
Algorithm~\ref{alg:success} outputs a subspace embedding with probability at least $1-\delta$.
In expectation Step 3 is run only a constant number of times with expected time $O(d^3 r^2 / \epsilon^2)$.
\end{theorem}

\begin{proof}
For each $j$, $\MH_j \MA$ succeeds with probability $99/100$, meaning that for all $x$ we have
$\| \MH_j \MA x \|_2 = (1 \pm \epsilon/9) \| \MA x \|_2$.
Suppose for some $j \neq j'$,
$\MH_j \MA$ and $\MH_{j'} \MA$ are both successful.
By definition we have
$$\| \MH_j \MA x \|_2 = (1\pm \epsilon/3) \| \MH_{j'} \MA x \|_2$$ for all $x$.
Taking the SVD of the embeddings, this is equivalent to
$$\| \MD_j \MV_j^\top x \|_2 = (1\pm \epsilon/3) \| \MD_{j'} \MV_{j'}^\top x\|_2$$ for all $x$.
Making the change of variable $y := \MD_j \MV_j^\top x$, this is equivalent to
$$\|y\|_2 = (1\pm \epsilon/3) \| \MD_{j'} \MV_{j'}^\top \MV_j \MD_j^{-1}  y\|_2$$ for all $y$,
which is true if and only if all singular values of $\MD_{j'} \MV_{j'}^\top \MV_j \MD_j^{-1}$ are in $[1-\epsilon/3, 1+\epsilon/3]$. 

Conversely, if all singular values of $\MD_{j'} \MV_{j'}^\top \MV_j \MD_j^{-1}$ are in $[1-\epsilon/3, 1+\epsilon/3]$,
one can trace the steps backward to conclude that $\| \MH_j \MA x \|_2 = (1 \pm \epsilon/3) \| \MH_{j'}\MA x \|_2$ for all $x$.

Since with probability at least $1-\delta$, a $9/10$ fraction of the embeddings succeed with accuracy $\epsilon/9$,
there exists a $j$ that can pass the test.
It follows that any index $j$ which passes the test in the algorithm with a majority of the $j' \neq j$
is a successful subspace embedding with accuracy $\epsilon$.

Moreover, if we choose a random $j$ to compare to the remaining $j'$, the expected number of choices of $j$ until the test passes is only constant.
Then finding the index $j$ only takes an expected $O(r)$ SVDs.

The time to do the SVD naively is $O(d^4/\epsilon^2)$.
We can improve this by letting $\MT$ be a fast Johnson-Lindenstrauss transform matrix of dimension $O(d r / \epsilon^2) \times O(d^2/\epsilon^2)$,
then we can replace $\MH_j \MA$ with $\MT \MH_j \MA$ for all $j \in [d]$.
Then the verification procedure would only take $O(d^3 r^2 / \epsilon^2)$ time.
\end{proof}

\subsection{Proofs for Randomized SVD}

\begin{algorithm}[t]
\caption{Randomized SVD~\citep{halko2011finding}}
\label{alg:randSVD}
\begin{algorithmic}[1]
\REQUIRE{matrix $\MA \in \R^{\ell \times d}$; parameters $t, q \in \N_+$.}
\STATE{$\mybullet$ {Stage A}}
\STATE{Generate an $\ell\times 2t$ Gaussian test matrix $\MOmega$.}
\STATE{Set $\MY=(\MA^\top \MA)^q \MA^\top\MOmega$, and compute QR-factorization: $\MY=\MQ\MR$.}
\STATE{\quad $\mybullet$ dimension: $[\MA]_{\ell \times d}, [\MOmega]_{\ell\times 2t}, [\MY]_{d\times 2t}, [\MQ]_{d\times 2t}$}
\STATE{$\mybullet$ {Stage B}}
\STATE{Set $\MB=\MA\MQ$, and compute SVD: $\MB=\MU\MD\tilde{\MV}^\top$.}
\STATE{Set $\MV=\MQ\tilde{\MV}$.}
\STATE{\quad $\mybullet$ dimension: $[\MB]_{\ell \times 2t}, [\MU]_{\ell\times \ell}, [\MD]_{\ell\times 2t}, [\tilde{\MV}]_{2t\times 2t}, [\MV]_{d\times 2t}$}
\ENSURE{$\MD, \MV$.}
\end{algorithmic}
\end{algorithm}

The details of randomized SVD are presented in Algorithm~\ref{alg:randSVD},
rephrased in our notations. We have the following analog of Lemma~\ref{lem:pca_app}.

\begin{lemma}\label{lem:randSVD}
Let $\MA \in \R^{\ell \times d}$ be an $\ell \times d$ matrix ($\ell>d$).
Let $\epsilon \in (0,1]$, $k, t \in \N_+$ with
$d-1 \geq t \geq k + \lceil 6k/\epsilon^2 \rceil -1$.
Let $\tcd{\MA}=\MA \MV \MV^\top$ where $\MV$ is computed by Algorithm~\ref{alg:randSVD} with $q=O(\log\max\{\ell,d\})$.
Then with probability at least $1-3e^{-t}$,
for any matrix $\MX$ with $d$ rows and $\|\MX\|_F^2 \leq k$, we have
\begin{eqnarray*}
\|(\MA - \tcd{\MA} ) \MX\|_F^2 & \leq & \frac{\epsilon^2}{3} \sum_{i=k+1}^d \sigma_i^2(\MA),\\
\bigl| \|\MA \MX\|_F^2 - \|\tcd{\MA} \MX\|_F^2 \bigr| & \leq & \epsilon \sum_{i=k+1}^d \sigma_i^2(\MA) + 2\epsilon\|\MA\MX\|_F^2.
\end{eqnarray*}
The algorithm runs in time $O(q t \ell d + t^2 (\ell + d))$.
\end{lemma}
\begin{proof}
As stated in Section 10.4 in~\citep{halko2011finding}, with probability at least $1-3e^{-t}$, we have
\begin{eqnarray}
\|\MA - \tcd{\MA}\|_S \leq 2\sigma_{t+1}(\MA). \label{eqn:randSVD_bound}
\end{eqnarray}
Then we have
\begin{eqnarray*}
\|(\MA - \tcd{\MA}) \MX\|_F^2 \leq \|\MX\|_F^2 \|\MA - \tcd{\MA}\|_S^2 \leq 2k\sigma^2_{t+1}(\MA)
\end{eqnarray*}
where the first inequality follows because the spectral norm is consistent with the Euclidean norm,
and the second inequality follows from (\ref{eqn:randSVD_bound}).
For our choice of $t$, we have
$$k \sigma^2_{t+1}(\MA) \leq \frac{\epsilon^2}{6} (t-k+1)\sigma^2_{t+1}(\MA)
\leq \frac{\epsilon^2}{6} \sum_{i=k+1}^{t+1}\sigma^2_i(\MA) \leq \frac{\epsilon^2}{6} \sum_{i=k+1}^{d}\sigma^2_i(\MA) \leq \frac{\epsilon^2}{6} d^2(\MA,\lspan{\MX}),$$
which leads to the first claim in the lemma.

To prove the second claim, first note that
$$\bigl| \|\MA \MX\|_F - \|\tcd{\MA} \MX\|_F \bigr|^2 \leq \|(\MA - \tcd{\MA}) \MX\|_F^2 \leq \frac{\epsilon^2}{3} d^2(\MA,\lspan{\MX}).$$
Then by Fact~\ref{fac:weakTri}, we have
$$\bigl| \|\MA \MX\|_F^2 - \|\tcd{\MA} \MX\|_F^2 \bigr| \leq \frac{3}{\epsilon} \bigl| \|\MA \MX\|_F - \|\tcd{\MA} \MX\|_F \bigr|^2  + 2\epsilon\|\MA\MX\|_F^2 \leq \epsilon d^2(\MA,\lspan{\MX}) + 2\epsilon\|\MA\MX\|_F^2$$
which completes the proof.
\end{proof}

\subsection{Proof of Theorem~\ref{thm:fastDisPCA_coreset}}
Let $\MT$ to be a diagonal block matrix with $\MH_1, \MH_2, \dots, \MH_s$ on the diagonal. Then Algorithm~\ref{alg:fastDisPCA} is just to run Algorithm {\sf disPCA} on $\MT\MP$ to get the principal components $\MV$. Recall that the goal is to show $\gpca{\MP} = \MP \MV \MV^\top$ is a good proxy for the original data $\MP$ with respect to $\ell_2$ error fitting problems. It suffices to show that  $\gpca{\MP}$ satisfies enjoys properties similar to those stated in Lemma~\ref{thm:DisPCA_P_l2}.

To prove this, we begin with a lemma saying that $\MT\gpca{\MP}$ enjoys such properties, i.e.\ such properties are approximately preserved when replacing exact SVD with randomized SVD in Algorithm {\sf disPCA} (Lemma~\ref{lem:fastDisPCA_TP_l2}). 
Then we can show that $\gpca{\MP}$ enjoys similar properties as $\MT\gpca{\MP}$,
i.e.\ these properties are approximately preserved under subspace embedding (Lemma~\ref{lem:fastDisPCA_P_l2}).

\begin{lemma}\label{lem:fastDisPCA_TP_l2}
For any $d \times k$ matrix $\MX$ with orthonormal columns,
\begin{eqnarray*}
\|\MT\MP \MX - \MT\gpca{\MP} \MX\|_F^2  & \leq & O(\epsilon^2) d^2(\MT\MP, \lspan{\MX}) + O(\epsilon^3) \|\MT\MP\MX\|_F^2,\\
\left| \|\MT\MP \MX\|_F^2 - \|\MT\gpca{\MP} \MX\|_F^2 \right| & \leq & O(\epsilon) d^2(\MT\MP, \lspan{\MX}) + O(\epsilon) \|\MT\MP\MX\|_F^2.
\end{eqnarray*}
\end{lemma}
\begin{proof}
The proof follows that of Lemma~\ref{thm:DisPCA_P_l2} to $\MT\MP$. 
But now exact SVD is replaced with randomized SVD, so we need to argue that randomized SVD produces similar result as exact SVD in the sense of  Lemma~\ref{lem:pca_app_gen}. This is already proved in Lemma~\ref{lem:randSVD}.
Also note that we need a technical lemma bounding the small error terms incurred on the intermediate result $\MT\tcd{\MP}$. This is done by Lemma~\ref{lem:fastDisPCA_smallTerms}.
\end{proof}

\begin{lemma}\label{lem:fastDisPCA_smallTerms}
\begin{eqnarray*}
\|\MT\tcd{\MP}\MX\|_F^2 \leq \epsilon d^2(\MT\MP,\lspan{\MX}) + (1+2\epsilon) \|\MT\MP\MX\|_F^2,\\
d^2(\MT\tcd{\MP},\lspan{\MX}) \leq (1+\epsilon)  d^2(\MT\MP,\lspan{\MX})+ \epsilon \|\MT\MP\MX\|_F^2.
\end{eqnarray*}
\end{lemma}
\begin{proof}
For the first statement, by Lemma~\ref{lem:randSVD}, we have
\begin{eqnarray}
\left| \|\MT\tcd{\MP}\MX\|_F^2 - \|\MT\MP\MX\|_F^2 \right|& \leq & \sum_{i=1}^s \left| \|\MT\MP_i \MX\|_F^2 - \|\MT\tcd{\MP}_i \MX\|_F^2 \right| \nonumber\\
& \leq & \epsilon \sum_{i=1}^s d^2(\MT\MP_i,\lspan{\MX}) + 2\epsilon \sum_{i=1}^s\|\MT\MP_i\MX\|_F^2 \nonumber\\
& \leq & \epsilon d^2(\MT\MP,\lspan{\MX}) + 2\epsilon \|\MT\MP\MX\|_F^2. \label{eqn:fastDisPCA_smallTerms_b1}
\end{eqnarray}
For the second statement, by Pythagorean Theorem,
\begin{eqnarray*}
d^2(\MT\tcd{\MP}, \lspan{\MX}) -  d^2(\MT\MP, \lspan{\MX})  & = & \left[\|\MT\tcd{\MP}\|_F^2 - \|\MT\tcd{\MP}\MX\|_F^2\right] - \left[\|\MT\MP\|_F^2 - \|\MT\MP\MX\|_F^2\right]\\
& = & \left[\|\MT\tcd{\MP}\|_F^2 - \|\MT\MP\|_F^2  \right] + \left[\|\MT\MP\MX\|_F^2 - \|\MT\tcd{\MP}\MX\|_F^2\right]\\
& \leq & \|\MT\MP\MX\|_F^2 - \|\MT\tcd{\MP}\MX\|_F^2.
\end{eqnarray*}
The second statement then follows from the last inequality and (\ref{eqn:fastDisPCA_smallTerms_b1}).
\end{proof}

\begin{lemma}\label{lem:fastDisPCA_P_l2}
For any $d \times k$ matrix $\MX$ with orthonormal columns,
\begin{eqnarray*}
\|\MP \MX - \gpca{\MP} \MX\|_F^2  & \leq & O(\epsilon^2) d^2(\MP, \lspan{\MX}) + O(\epsilon^3) \|\MP\MX\|_F^2,\\
\left| \|\MP \MX\|_F^2 - \|\gpca{\MP} \MX\|_F^2 \right| & \leq & O(\epsilon) d^2(\MP, \lspan{\MX}) + O(\epsilon) \|\MP\MX\|_F^2.
\end{eqnarray*}
\end{lemma}
\begin{proof}
By the property of subspace embedding, we have $\|\MT\MP \MX - \MT\gpca{\MP} \MX\|_F^2 = (1\pm\epsilon)\|\MP \MX - \gpca{\MP} \MX\|_F^2$,
$\|\MT\MP\MX\|_F^2=(1\pm \epsilon)\|\MP\MX\|_F^2$
and $ d^2(\MT\MP, \lspan{\MX}) = \|\MT\MP - \MT\MP\MX\MX^\top\|_F^2 = (1\pm \epsilon) \|\MP - \MP\MX\MX^\top\|_F^2 = (1\pm \epsilon) d^2(\MP, \lspan{\MX})$.
Then
\begin{eqnarray*}
(1+\epsilon)\|\MP \MX - \gpca{\MP} \MX\|_F^2  & \leq & \|\MT\MP \MX - \MT\gpca{\MP} \MX\|_F^2 \\
& \leq & O(\epsilon^2) d^2(\MT\MP, \lspan{\MX}) + O(\epsilon^3) \|\MT\MP\MX\|_F^2 \\
& \leq  & O(\epsilon^2) d^2(\MP, \lspan{\MX}) + O(\epsilon^3) \|\MP\MX\|_F^2
\end{eqnarray*}
where the second inequality is from Lemma~\ref{lem:fastDisPCA_TP_l2}. This then leads to the first statement.

For the second statement, we have
\begin{eqnarray*}
(1+\epsilon)\|\MP \MX\|_F^2 - (1-\epsilon)\|\gpca{\MP} \MX\|_F^2  & \leq & \|\MT\MP \MX\|_F^2 - \|\MT\gpca{\MP} \MX\|_F^2 \\
& \leq & O(\epsilon) d^2(\MT\MP, \lspan{\MX}) + O(\epsilon) \|\MT\MP\MX\|_F^2 \\
& \leq  & O(\epsilon) d^2(\MP, \lspan{\MX}) + O(\epsilon) \|\MP\MX\|_F^2
\end{eqnarray*}
which leads to
\begin{eqnarray*}
\|\MP \MX\|_F^2 - \|\gpca{\MP} \MX\|_F^2 \leq O(\epsilon) d^2(\MP, \lspan{\MX}) + O(\epsilon) \|\MP\MX\|_F^2.
\end{eqnarray*}
A similar argument bounds $\|\gpca{\MP} \MX\|_F^2 - \|\MP \MX\|_F^2 $, which completes the proof.
\end{proof}

We represent Theorem~\ref{thm:fastDisPCA_coreset} in a general form for $\ell_2$-error geometric fitting problems.

\noindent
\textbf{Theorem~\ref{thm:fastDisPCA_coreset}.}
{\it
Suppose Algorithm~\ref{alg:fastDisPCA} takes $\epsilon \in (0,1/2]$, $t_1=t_2 = O(\max\cbr{ \frac{k}{\epsilon^2}, \log \frac{s}{\delta}}), \ell = {O}(\frac{d^2}{\epsilon^2}), q=O(\max\{\log \frac{d}{\epsilon},\log \frac{sk}{\epsilon}\})$ as input, and sets the failure probability of each local subspace embedding to $\delta' = \delta/2s$.
Let $\gpca{P}=\MP\MV \MV^\top$. 
Then with probability at least $1-\delta$,
there exists a constant $c_0 \geq 0$, such that
for any set of $k$ points $\x$,
\begin{eqnarray*} (1-\epsilon) d^2(\MP,\x) - \epsilon \|\MP\MX\|_F^2 \leq  d^2(\gpca{P}, \x) + c_0 \leq  (1+\epsilon) d^2(\MP,\x) + \epsilon \|\MP\MX\|_F^2
\end{eqnarray*}
where $\MX$ is an orthonormal matrix whose columns span $\x$.
The total communication is $O(skd/\epsilon^2)$ and the total time is
$O\left(\nnz(\MP) + s\left[\frac{d^3k}{\epsilon^4} + \frac{k^2 d^2}{\epsilon^6}\right] \log \frac{d}{\epsilon}\log\frac{sk}{\delta\epsilon}\right)$.
}

\begin{proof}
The proof of correctness follows the proof of Theorem~\ref{thm:coreset_gen}, replacing the use of Lemma~\ref{thm:DisPCA_P_l2} with Lemma~\ref{lem:fastDisPCA_P_l2}.

On each node $v_i$, the subspace embedding takes time $O(\nnz(\MP_i))$, and the randomized SVD takes time $O(q t_1\ell d+t_1^2(\ell +d))$;
on the central coordinator, the randomized SVD takes time $O(q t_1 (s t_1) d + t_1^2(s t_1+d))$ since $\MY$ has $O(s t_1)$ non-zero rows.
The total running time then follows from the choice of the parameters.
The total communication cost follows from the fact that the algorithm only sends $\tc{\MD_i}{t_1}, \tc{\MV_i}{t_1}$ from each node to the central coordinator.
\end{proof}

\end{document}

%% file: lowRank.tex
\begin{tikzpicture}
\usetikzlibrary{calc}

\node (v1) at (-3.4,-0.2) {};
\node (v2) at (2,2) {};
\node (v4) at (1.8,-2.6) {};
\draw  (v1) edge (v2);
\node (v3) at (-3.2,0.4) {};
\draw  (v3) edge (v4);
\node (p) at (1.2,0) {};
\node (pp) at (1.8,-0.4) {};
\fill[blue]  (p) ellipse (0.15 and 0.15);
\fill[blue]  (pp) ellipse (0.15 and 0.15);

\draw (p) edge ($(v1)!(p)!(v2)$);
\draw (p) edge ($(v3)!(p)!(v4)$);
\draw (pp) edge ($(v1)!(pp)!(v2)$);
\draw (pp) edge ($(v3)!(pp)!(v4)$);

\node at (0.6,0) {\textbf{$P$}};
\node at (2.4,-0.4) {\textbf{$\widehat{P}$}};
\node at (1,2.2) {{$L_{\textbf{X}^*}$}};
\node at (0.8,-2.6) {{$L_{\textbf{V}^{(r)}}$}};

\node at (0.8,-0.8) {$a$};
\node at (0.8,0.8) {$b$};
\node at (1.4,-1.2) {$c$};
\node at (1.4,0.8) {$d$};
\end{tikzpicture}

%% file: disPCAMotivate2.tex
\begin{tikzpicture}

\node (ph) at (-1.2,1.6) {};
\node (p) at (-0.6,2.2) {};
\node (phx) at (-1.2,-1.2) {};
\node (px) at (-0.6,-1.2) {};
\node (c) at (3.2,-1.2) {};

\coordinate (v1) at (-3.4,-1.2);
\coordinate (v2) at (4.8,-1.2);

\fill (ph) circle (.1) node [above=2]{\textcolor{violet}{$\tilde{P}$} };
\fill (p) circle (.1) node [above=2]{\textcolor{blue}{$P$}};
\fill (c) circle (.1);
\draw[violet]  (ph) edge (c);
\draw[blue]  (p) edge (c);
\node at (3.2,-1.2) [above=1]{$\mathcal{L}$}; 
\draw  (v1) edge (v2);
\node at (4.6,-1.6) {$L_\mathbf{X}=\text{span}(\mathcal{L})$};

\fill (px) circle (.1);
\draw[blue]  (p) edge (px);

\fill (phx) circle (.1);
\draw[violet]  (ph) edge (phx);
 
\end{tikzpicture}